\documentclass[lettersize,journal]{IEEEtran}
\usepackage{hyperref}
\usepackage{amsmath,amsfonts}
\usepackage{algorithm}
\usepackage{array}
\usepackage[caption=false,font=normalsize,labelfont=sf,textfont=sf]{subfig}
\usepackage{textcomp}
\usepackage{stfloats}
\usepackage{url}
\usepackage{verbatim}
\usepackage{graphicx}
\usepackage{cite}
\usepackage[table,xcdraw]{xcolor}

\usepackage{booktabs}
\usepackage{amssymb}
\usepackage{mathtools}
\usepackage{algorithm}
\usepackage{algpseudocode}
\usepackage{float}
\usepackage{multirow}

\usepackage{amsthm}
\theoremstyle{plain}
\newtheorem{assumption}{Assumption}
\newtheorem{proposition}{Proposition}
\newtheorem{remark}{Remark}
\newtheorem{definition}{Definition}
\newtheorem{theorem}{Theorem}

\usepackage[normalem]{ulem}
\useunder{\uline}{\ul}{}

\usepackage{amsmath,amsfonts,bm}
\usepackage{xspace}

\def\eqref#1{equation~\ref{#1}}

\def\1{\bm{1}}

\def\eps{{\epsilon}}

\DeclareMathAlphabet{\mathsfit}{\encodingdefault}{\sfdefault}{m}{sl}
\SetMathAlphabet{\mathsfit}{bold}{\encodingdefault}{\sfdefault}{bx}{n}

\def\gO{{\mathcal{O}}}

\def\gU{{\mathcal{U}}}

\DeclareMathOperator*{\argmin}{arg\,min}

\newcommand{\x}{{\boldsymbol x}}
\newcommand{\kb}{{\boldsymbol k}}
\newcommand{\s}{{\boldsymbol s}}
\newcommand{\y}{{\boldsymbol y}}
\newcommand{\z}{{\boldsymbol z}}
\newcommand{\n}{{\boldsymbol n}}
\newcommand{\m}{{\boldsymbol m}}
\newcommand{\vbb}{{\boldsymbol v}}

\newcommand{\w}{{\boldsymbol w}}

\newcommand{\Ib}{{\boldsymbol I}}

\newcommand{\g}{{\boldsymbol g}}

\newcommand{\Rd}{{\mathbb R}}

\newcommand{\Ed}{{\mathbb E}}

\newcommand{\Ac}{{\mathcal A}}

\newcommand{\Nc}{{\mathcal N}}

\definecolor{C0}{rgb}{0.121569, 0.466667, 0.705882}
\definecolor{C1}{rgb}{1.000000, 0.498039, 0.054902}
\definecolor{C2}{rgb}{0.172549, 0.627451, 0.172549}
\definecolor{C3}{rgb}{0.839216, 0.152941, 0.156863}
\definecolor{C4}{rgb}{0.580392, 0.403922, 0.741176}
\definecolor{C5}{rgb}{0.549020, 0.337255, 0.294118}
\definecolor{C6}{rgb}{0.890196, 0.466667, 0.760784}
\definecolor{C7}{rgb}{0.498039, 0.498039, 0.498039}
\definecolor{C8}{rgb}{0.737255, 0.741176, 0.133333}
\definecolor{C9}{rgb}{0.090196, 0.745098, 0.811765}

\definecolor{myPurple}{rgb}{0.4, .0, .8}
\definecolor{myGreen}{rgb}{0, .8, .3}
\definecolor{myRed}{rgb}{0.8, .2, .2}
\definecolor{myOrange}{rgb}{0.7, 0.45, 0.2}
\definecolor{myBlue}{rgb}{.0, .0, 1.0}
\definecolor{myBlue2}{rgb}{.0, .0, 0.5}
\definecolor{myBlack}{rgb}{.0, .0, 0.0}

\begin{document}

\title{Fast and Stable Diffusion Inverse Solver with History Gradient Update}

\author{Linchao He, Hongyu Yan, Mengting Luo, Hongjie Wu, Kunming Luo, Wang Wang, Wenchao Du, Hu Chen~\IEEEmembership{Member,~IEEE,}, Hongyu Yang, Yi Zhang~\IEEEmembership{Senior Member,~IEEE,}, Jiancheng Lv~\IEEEmembership{Senior Member,~IEEE,}
\thanks{Corresponding author: Hu Chen e-mail: huchen@scu.edu.cn). Linchao He and Hongyu Yan have equal contribution.}
\thanks{Linchao He and Mengting Luo are with the Department of National Key Laboratory of Fundamental Science on Synthetic Vision, Sichuan University, Chengdu 610065, China. (e-mail: hlc@stu.scu.edu.cn).}
\thanks{Yi Zhang is with the School of Cyber Science and Engineering, Sichuan University, Chengdu 610065, China (e-mail: yzhang@scu.edu.cn).}
\thanks{Linchao He, Mengting Luo, Hongjie Wu, Wenchao Du, Hu Chen, Jiancheng Lv, and Hongyu Yang are with the College of Computer Science, Sichuan University, Chengdu 610065, China. (e-mail:  \{hlc, lmt\}@stu.scu.edu.cn; \{wuhongjie0818, wenchaodu.scu\}@gmail.com; \{huchen, lvjiancheng\}@scu.edu.cn; yanghongyu\_scu@163.com).}
\thanks{Hongyu Yan and Kunming Luo are with the Hong Kong University of Science and Technology, Hong Kong, China. (e-mail: \{hyanar, kluoad\}@connect.ust.hk)}
\thanks{Wang Wang is with the Institute of Space and Earth Information Science, The Chinese University of Hong Kong, Shatin, N.T., Hong Kong, China. (e-mail: wangw00821@gmail.com)}}

\markboth{Journal of \LaTeX\ Class Files,~Vol.~14, No.~8, August~2021}%
{Shell \MakeLowercase{\textit{et al.}}: A Sample Article Using IEEEtran.cls for IEEE Journals}


\maketitle

\begin{abstract}
Diffusion models have recently been recognised as efficient inverse problem solvers due to their ability to produce high-quality reconstruction results without relying on pairwise data training. Existing diffusion-based solvers utilize Gradient Descent strategy to get a optimal sample solution.
However, these solvers only calculate the current gradient and have not utilized any history information of sampling process, thus resulting in unstable optimization progresses and suboptimal solutions. To address this issue, we propose to utilize the history information of the diffusion-based inverse solvers. In this paper, we first prove that, in previous work, using the gradient descent method to optimize the data fidelity term is convergent. Building on this, we introduce the incorporation of historical gradients into this optimization process, termed History Gradient Update (HGU). We also provide theoretical evidence that HGU ensures the convergence of the entire algorithm. It's worth noting that HGU is applicable to both pixel-based and latent-based diffusion model solvers. Experimental results demonstrate that, compared to previous sampling algorithms, sampling algorithms with HGU achieves state-of-the-art results in medical image reconstruction, surpassing even supervised learning methods. Additionally, it achieves competitive results on natural images.
\end{abstract}


\begin{figure*}[t]
	\centering
	\includegraphics[width=1\linewidth]{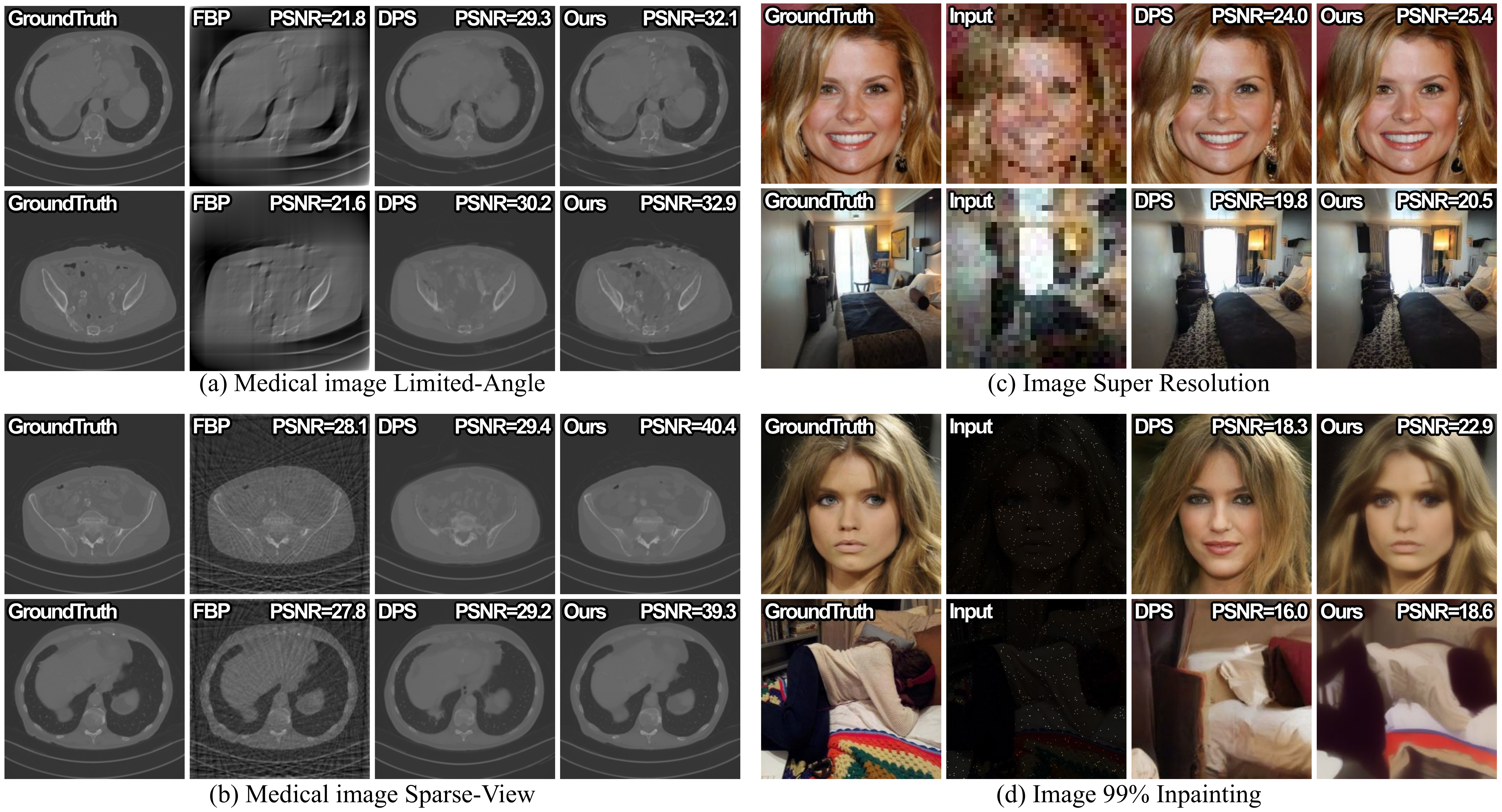}
	\caption{Our method is capable of reconstructing both nature images and medical data in a zero-shot manner with stable and fast optimization processes. In this paper, we demonstrate the reconstruction ability of our method on various datasets and measurements. FBP is filtered backprojection.}
	\label{fig:coverfig}
\end{figure*}

\section{Introduction}
Inverse problem solving is of paramount significance, as it typically entails the recovery of missing data from sparse measurement $\y$ which is usually formulated as:
\begin{align}
	\hat{\x} = \argmin_\x \left|\left|\Ac \x - \y\right|\right|_2^2 + \lambda R\left(\x\right),
	\label{eq:ir_formualtion}
\end{align}
where $\left|\left|\Ac \x - \y\right|\right|_2^2$ is the data-fidelity term that ensures the reconstructed results are consistent with the measurements, while $\lambda R\left(\x\right)$ is a prior term that ensures the reconstructed results are realistic and follow the ground truth image distribution $p\left(\x\right)$. Due to the sparsity of $\y$, the solving processes are well-known ill-posed problems. Therefore, it is important to take advantage of good prior models to generate high-quality results. Commonly, pre-trained generative models are used as the prior models, including GANs~\cite{bora2017compressed,goodfellow2014generative}, VAEs~\cite{bora2017compressed,kingma2013auto}, and diffusion models~\cite{ho2020denoising,chung2022improving,song2021solving}.

Previous research on optimization problems in diffusion-based inverse problem solvers, specifically with regards to optimizing the data fidelity term, has traditionally utilized the most basic Gradient Descent (GD) strategy~\cite{chung2022improving,chung2022solving,chung2022diffusion,Song2023PseudoinverseGuidedDM} to obtain the generated result. However, the loss-guided optimization process of the data fidelity term is typically a stochastic process that is naturally unstable. Additionally, the GD strategy only relies on the gradient in a single step to optimize the target. These facts make the current diffusion-based solvers difficult to optimize and ultimately lead to inaccurate samples. One solution to this issue is to incorporate historical gradient information into the optimization processes, which can make the processes more stable and produce high-quality samples. However, current research has yet to explore the historical gradient information, and no studies have provided evidence of the convergence properties of the GD strategy when optimizing the data fidelity term.


Pixel-based diffusion models have been widely used to solve inverse problems in recent research~\cite{lugmayr2022repaint,song2021solving,chung2022diffusion,kawar2022denoising,chung2022improving,wang2022zero,Song2023PseudoinverseGuidedDM}. These models can effectively act as prior models and generate high-quality samples in an iterative optimization framework. However, they have a significant limitation: they can learn accurate prior information from large-scale data, but inaccurate from small-scale data. This limits their efficiency and applicability to data-constrained tasks, such as medical image inverse problems. To the best of our knowledge, no previous work has investigated the applicability of diffusion-based solvers on small-scale datasets.


In this paper, we aim to address the above issues by introducing a new optimization method, which greatly expands the applicability of diffusion-based solvers. Specifically, \textbf{(1)} to address the optimization problems within diffusion-based solvers, we prove the evidence of the convergence properties of the GD strategy. \textbf{(2)} Based on this evidence, we develop a new strategy for the optimization of data fidelity term which use the historical gradient from previous optimization steps to adaptively adjust sample-level gradient, namely History Gradient Update (HGU), thereby stabilizing the whole optimization process and improving the quality of final samples. \textbf{(3)} We introduce the latent diffusion model (LDM) to the diffusion-based solver, such that we can learn accurate prior information from small-scale data.


Extensive experiments of various inverse problems demonstrate that HGU with LDMs outperforms state-of-the-art supervised and unsupervised methods on small-scale medical datasets. Additionally, as a general method of inverse problem solvers, we extend HGU to the natural image restoration task and HGU achieves competitive performance compared to other state-of-the-art zero-shot methods. HGU provides a valuable optimization tool for solving the inverse problem, allowing the community to solve the inverse problem for different modalities without regrading the scale of the dataset and leveraging the vast amount of available LDMs such as Stable Diffusion. Fig.~\ref{fig:coverfig} shows some representative visual results of the proposed method.

\section{Background}

\subsection{Diffusion models}
Consider a Gaussian diffusion process, where $\x_t \in \Rd^n, t\in[0,T]$ and initial $\x_0$ is sampled from the original data distribution $P_\text{data}$. We define the forward diffusion process using stochastic differential equation (SDE)~\cite{song2021scorebased}:
\begin{align}
	d\x = f\left(\x,t\right) \, dt + g\left(t\right) \, d\w,
	\label{eq:forward_sde}
\end{align}
where $f\left(\cdot,t\right): \, \mathbb{R}^n \to \mathbb{R}^n > 0$ is a drift coefficient function, $g\left(t\right) \in \mathbb{R}$ is defined as a diffusion coefficient function, and $\w \in \mathbb{R}^n$ is a standard $n$-dimensional Brownian motion. Thus, the reverse SDE of Eq.~\eqref{eq:forward_sde} can also defined as:
\begin{align}
	d\x = \left[f\left(\x_t,t\right)-g\left(t\right)^2\nabla_{\x_t} \log p_t\left(\x_t\right)\right] \, dt + g\left(t\right) \, d\bar{\w},
	\label{eq:reverse_sde}
\end{align}
where $dt$ is a negative infinitesimal time step and $d\bar{\w}$ is the backward process of $d\w$. The reverse SDE defines a generative process that transforms standard Gaussian noise into meaningful content. To accomplish this transformation, the score function $\nabla_{\x_t} \log p_t\left(\x_t\right)$ needs to be matching, which is typically replaced with $\nabla_{\x_t} \log p_{0 | t}\left(\x_t \middle| \x_0\right)$ in practice. Therefore, we can train a score model $\s_\theta\left(\x_t,t\right)$, so that $\s_\theta\left(\x_t,t\right) \approx \nabla_{\x_t} \log p_t\left(\x_t\right) \approx \nabla_{\x_t} \log p_{0 | t}\left(\x_t \middle| \x_0\right)$ using the following score-matching objective:
\begin{multline}
	\label{eq:sde_objective}
	\min_\theta \Ed_{t \in \left[1,\dots,T\right], \x_0 \sim P_\text{data}, \x_t \sim p_{0|t}\left(\x_t \middle| \x_0 \right)} \\ \left[ \left|\left| \s_\theta\left(\x_t,t\right)- \nabla_{\x_t} \log p_{0|t}\left(\x_t \middle| \x_0 \right) \right|\right| _2^2\right] \, .
\end{multline}
Subsequently, the reverse SDE can yield meaningful contents $\x_0 \sim P_\text{data}$ from random noises $\x_{T} \sim \Nc(\bm{0}, \Ib)$ by iteratively using $\s_\theta \left(\x_t,t\right)$ to estimate the scores $\nabla_{\x_t} \log p_t\left(\x_t\right)$. In our experiments, we adopt the standard Denoising Diffusion Probabilistic Models (DDPM)~\cite{ho2020denoising} which is equivalent to the above variance preserving SDE (VP-SDE)~\cite{song2021scorebased}.

\subsection{Diffusion model for inverse problem solving}
To solve the inverse problems using the diffusion model, various workarounds are proposed~\cite{rombach2022high,saharia2022image,gao2023implicit,luo2023image}. These methods use conditional diffusion models to iteratively denoise Gaussian noise and obtain reconstructed samples. However, these approaches have limitations, as they rely on conditional diffusion models that require paired data for training and can only handle specific tasks without retraining. To address these issues, several zero-shot diffusion-based inverse solvers~\cite{lugmayr2022repaint,song2021solving,kawar2022denoising,chung2022improving,wang2022zero} have been proposed. Typically, for each denoising step, they~\cite{lugmayr2022repaint,song2021solving,wang2022zero} unconditionally estimate new denoised samples based on the previous step, followed by replacing the corresponding items in the denoised samples using the measurement $\y$, which is also known as range-null space decomposition~\cite{wang2022zero}. This approach ensures data consistency, but it fails in the case of noisy measurements, since $\Ac^{-1} \y$ is not a correct corresponding item for the denoised samples. To deal with noisy measurements, alternative methods~\cite{chung2022diffusion,chung2022parallel} have been proposed to solve the inverse problems with noised measurements. Rather than directly replace items, these approaches use the gradient of $\left|\left|\y-\Ac \x \right|\right|_2^2$ to conditionally guide the generative process. These methods are robust to noise and can process nonlinear projection operators, which can be formulated as
\begin{align}
	& \nabla_{\x_t} \log p(\y|\x_t) \simeq \nabla_{\x_t} \log p\left(\y|\Ed \left[\x_0|\x_t\right] \right), \nonumber \\
	& \Ed \left[\x_0|\x_t\right] = \frac{1}{\sqrt{\bar{\alpha}}} \left(\x_t-(1-\bar{\alpha})\s_\theta(\x_t,t)\right).
	\label{eq:dps}
\end{align}

Despite the achievements, these methods try to solve the inverse problem on the pixel space which are very effective when having large-scale training dataset for DDPM models, while obtaining bad performance on small-scale datasets which limits their availability in solving the inverse problem of medical imaging.

\section{Method}

\subsection{Latent diffusion solver}


Previous works on diffusion-based inverse problem solving, such as~\cite{chung2022diffusion,chung2022improving,chung2022parallel,chung2022solving,song2021solving,wang2022zero}, have solved the problems in the pixel space, which is sensitive to the scale of data. To overcome this limitation, we are inspired by the Latent Diffusion Models (LDMs)~\cite{rombach2022high}. We introduce a new inverse solver called Latent Diffusion Solver (LDS) that aims to solve general inverse problems on the latent space for small-scale datasets.
\begin{assumption}\label{autoencoder}
For paired encoder $\mathcal{E}$ and decoder $\mathcal{D}$, $\mathcal{E}$ can compress any data $\x \sim p_\text{data}$ into a distinct low-dimensional latent $\z \sim p_\text{latent}$, and subsequently $\mathcal{D}$ can restore $\z$ from $\x$.
\begin{equation}
\z = \mathcal{E}(\x), \x = \mathcal{D}(\z), \; \; \; x \sim p_\text{data}
\end{equation}	
\end{assumption}
Although the above Assumption~\ref{autoencoder} is a strong assumption, in real-world applications, variational autoencoders~\cite{kingma2013auto,van2017neural} have revealed their strong performance in compressing and restoring. When both the encoder and decoder are well-trained, we posit that the aforementioned assumptions hold. So that we can obtain a score model on latent space denoted as $\epsilon^\z_\theta$ as shown in Proposition~\ref{scoremodel}.
\begin{proposition}\label{scoremodel}
	Considering a data $\x \sim p_\text{data}$, we can get its unique latent $\z$ by $\z = \mathcal{E}(\x)$. By minimizing the below score matching function, we can get $\epsilon^\z_\theta$:
	\begin{multline}
		\min_\theta \Ed_{t \in \left[1,\dots,T\right], \z_0 \sim p_\text{latent}, \z_t \sim p_{0|t}\left(\z_t \middle| \z_0 \right)} \\ \left[ \left|\left| \s^\z_\theta\left(\z_t,t\right)- \nabla_{\z_t} \log p_{0|t}\left(\z_t \middle| \z_0 \right) \right|\right| _2^2\right] \, , \z = \mathcal{E}(\x).
	\end{multline}
\end{proposition}
Also, given the well-trained decoder $\mathcal{D}$, we can have a new data consistency term on the latent space as shown in the following Proposition.
\begin{proposition}	\label{dataconsistency}
	Considering a latent $\z_t \sim p_\text{latent}(z_i|z_0)$ where $i \in [0, 1]$ and $\z_0 \sim p_\text{latent}$, the data consistency term on the latent space can be computed by 
	\begin{align}
		& p(\y| \z_i) = \gU (\y, \Ac \mathcal{D}\left(\z_i\right)).
	\end{align}
\end{proposition}
Based on the above assumption and proposition, we can solve inverse problems on the latent space without losing any generalization ability. Similarly, we can use the posterior mean method from DPS~\cite{chung2022diffusion} to replace $\z_i$ with $\hat{\z}_0$ to derive LDS. The following theorem indicates this manner.
\begin{theorem}[Latent Diffusion Solver]\label{eq:lds}
Considering Assumption~\ref{autoencoder}, Proposition~\ref{scoremodel} and~\ref{dataconsistency}, we can derive Eq.~\ref{eq:dps} to the latent space. Formally,
	\begin{align}
		\nabla_{\z_t} \log p(\y|\z_t) & \simeq \nabla_{\z_t} \log p\left(\y|\Ed \left[\z_0|\z_t\right] \right) \nonumber \\
		& =\gU\left( \y, \Ac \mathcal{D}\left(\Ed \left[\z_0|\z_t\right]\right) \right), \nonumber \\
		\Ed \left[\z_0|\z_t\right] &= \frac{1}{\sqrt{\bar{\alpha}}} \left(\z_t-(1-\bar{\alpha})\s^\z_\theta(\z_t,t)\right), \\
		\nabla_{\z_t} \log p(\z_t|\y) & \simeq \s^\z_\theta(\z_t,t) - \epsilon \nabla_{\z_t} \gU \left( \y, \Ac \mathcal{D}\left(\Ed \left[\z_0|\z_t\right]\right) \right)
	\end{align}
where $\epsilon$ is the guidance rate, which can balance the realness and consistency of results, and $\gU$ is the evaluation function which evaluates the difference between the measurement $\y$ and predicted measurement $\Ac \mathcal{D}\left(\Ed \left[\z_0|\z_t\right]\right)$.
\end{theorem}

\begin{figure*}[t]
	\centering
 \label{fig:figframework}
	\includegraphics[width=1\linewidth]{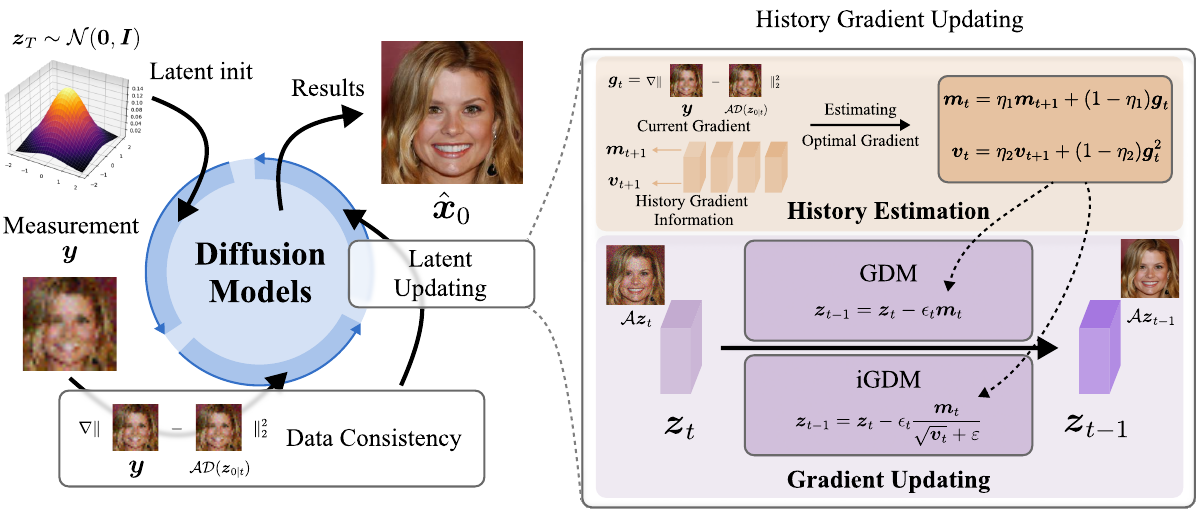}
    \caption{\textbf{Left: the proposed framework for inverse problem solving.}. We extend the previous work in the pixel-based diffusion models (~\cite{chung2022diffusion}) to the latent-based diffusion models, which are more efficient. We design a new approach to update the latent after every data consistency step, making the optimization process of latent more stable and the results more accuracy. \textbf{Right: an illustration of the latent updating process and the proposed History Gradient Updating.} The history gradients are essential to stabilize optimization processing and improve the quality of samples. We collect the history information from previous step, and estimate the optimal gradient factors $\{\m_t, \vbb_t\}$ based on the current gradient $\g_t$. Finally, we compute the optimal gradients based on the estimated momentum through our Momentum-variant HGU (GDM) and Improved-Momentum-variant HGU (iGDM) algorithms and obtain the next latent $\z_{t-1}$.}
\end{figure*}

\subsection{History gradient update}

\begin{figure*}
	\begin{minipage}[!h]{0.48\textwidth}
		\begin{algorithm}[H]
			\caption{HGU: Momentum-variant.}
			\label{algo:ldir_momentum}
			\begin{algorithmic}[1]
				\Require $T, y, \s_\theta^\z, \mathcal{D}, \eta$
				\State $\z_T \sim \Nc(\bm{0}, \Ib)$
				\For{$t = T-1, \cdots, 0$}
				\State $\z_{0|t} \gets \frac{1}{\sqrt{\bar{\alpha}_t}} \left(\z_t-(1-\bar{\alpha}_t)\s^\z_\theta(\z_t,t)\right)$
				\State $\kb \sim \Nc(\bm{0}, \Ib)$
				\State $\z_{t-1}' \gets \frac{\sqrt{\alpha_i}(1-\bar{\alpha}_{t-1})}{1-\bar{\alpha}_t}\z_t + \frac{\sqrt{\bar{\alpha}_{t-1}}\beta_t}{1-\bar{\alpha}_t} \z_{0|t} + g(t)\kb$
				\If{$t == T-1$}
				\State $\m_t \gets \nabla_{\z_t} \gU \left( \y, \Ac \mathcal{D}\left( \z_{0|t} \right) \right)$
				\Else
				\State $\m_t \gets \eta \m_{t+1} + (1-\eta) \nabla_{\z_t} \gU \left( \y, \Ac \mathcal{D}\left( \z_{0|t} \right) \right)$
				\EndIf
				\State $\z_{t-1} \gets \z_{t-1}' - \eps_t \m_t$ 
				\EndFor
				\State $\hat{\x}_0 \gets \mathcal{D}\left(\z_{0}\right)$
				\newline
				\Return $\hat{\x}_0$
			\end{algorithmic}
		\end{algorithm}
	\end{minipage}
	\hspace{0.04\textwidth}
	\begin{minipage}[!h]{0.48\textwidth}
		\begin{algorithm}[H]
			\caption{HGU: Improved-Momentum-variant.}
			\label{algo:ldir_adam}
			\begin{algorithmic}[1]
				\Require $T, y, \s_\theta^\z, \mathcal{D}, \eta_1, \eta_2, \varepsilon, \eps$
				\State $\z_T \sim \Nc(\bm{0}, \Ib), \m_{T} \gets 0, \vbb_{T} \gets 0$
				\For{$t = T-1, \cdots, 0$}
				\State $\z_{0|t} \gets \frac{1}{\sqrt{\bar{\alpha}_t}} \left(\z_t-(1-\bar{\alpha}_t)\s^\z_\theta(\z_t,t)\right)$
				\State $\kb \sim \Nc(\bm{0}, \Ib)$
				\State $\z_{t-1}' \gets \frac{\sqrt{\alpha_t}(1-\bar{\alpha}_{t-1})}{1-\bar{\alpha}_t}\z_t + \frac{\sqrt{\bar{\alpha}_{t-1}}\beta_t}{1-\bar{\alpha}_t} \z_{0|t} + g(t)\kb$
				\State $\g_t \gets \nabla_{\z_t} \gU \left( \y, \Ac \mathcal{D}\left( \z_{0|t} \right) \right)$
				\State $\m_t \gets \eta_1 \m_{t+1} + (1-\eta_1) \g_t$
				\State $\vbb_t \gets \eta_2 \vbb_{t+1} + (1 - \eta_2) \g_t^2$
				\State $\z_{t-1} \gets \z_{t-1}' - \eps \frac{\m_t}{\sqrt{\vbb_t} + \varepsilon}$ 
				\EndFor
				\State $\hat{\x}_0 \gets \mathcal{D}\left(\z_0\right)$
				\newline
				\Return $\hat{\x}_0$
			\end{algorithmic}
		\end{algorithm}
	\end{minipage}
\end{figure*}



Latent diffusion solvers typically generate superior samples~\cite{rombach2022high}. However, both pixel and latent-based diffusion solvers still face challenges from unstable optimization processes and the gradient descent strategy, which result in sub-optimal samples. To improve the stability of solvers and achieve the best quality samples, we propose a new optimization method for diffusion-based solvers, called History Gradient Update. In this section, we explain how our HGU solves the inverse problem with two variants.

\textbf{Momentum-variant. } Given a well-trained latent-based score model $\s^\z_\theta$, we have the posterior mean:
\begin{equation}
	\Ed \left[\z_0|\z_t\right] = \frac{1}{\sqrt{\bar{\alpha}}} \left(\z_t-\sqrt{1-\bar{\alpha}}\s^\z_\theta(\z_t,t)\right), t \in [0, T],
\end{equation}
where $\z_T \sim \Nc(\bm{0}, \Ib)$. Thus we can add historical gradient to the reconstruction as:
\begin{align}
	\z_{t-1} &= \z'_{t-1} - \eps_t \m^t, \\
	\z'_{t-1} &= \frac{\sqrt{\alpha_t}(1-\bar{\alpha}_{t-1})}{1-\bar{\alpha}_t}\z_t + \frac{\sqrt{\bar{\alpha}_{t-1}}\beta_t}{1-\bar{\alpha}_t}\Ed \left[\z_0|\z_t\right] + g(t)\kb ,\\
	\m_t &= \eta \m_{t-1} + (1-\eta) \nabla_{\z_t} \gU \left( \y, \Ac \mathcal{D}\left(\Ed \left[\z_0|\z_t\right]\right) \right),
\end{align}
where $\kb \sim \Nc(\bm{0}, \Ib)$. The final results can be obtained by decoding the latent $\z_0$ as $\hat{\x}_0 = \mathcal{D}(\z_0)$.

\textbf{Improved-Momentum-variant. } Similarly to the Momentum-variant, we have the same predicted posterior mean $\Ed \left[\z_0|\z_t\right]$. The Momentum-variant only considers the first momentum information. Thus, we develop a new variant, named Improved-Momentum-variant HGU (iGDM), which includes the both first momentum and second momentum information~\cite{kingma2014adam}. We derive the Improved-Momentum-variant HGU as:
\begin{align}
	\z_{t-1} &= \z'_{t-1} - \eps_t \frac{\m^t}{\sqrt{\vbb^t} + \varepsilon}, \\
	\m_t &= \eta_1 \m_{t-1} + (1 - \eta_1) \gU \left( \y, \Ac \mathcal{D}\left(\Ed \left[\z_0|\z_t\right]\right) \right), \\
	\vbb_t &= \eta_2 \vbb_{t-1} + (1 - \eta_2) \gU \left( \y, \Ac \mathcal{D}\left(\Ed \left[\z_0|\z_t\right]\right) \right)^2,
\end{align}
where $\varepsilon$ helps improve the numerical stability. The detail pseudocode of the above variant HGU is provided in Algorithm~\ref{algo:ldir_momentum} and~\ref{algo:ldir_adam}.


It should be noted that our iGDM and Adam~\cite{kingma2014adam} both use the first and second momentum to optimize variables. However, Adam is designed as an optimizer for neural network training, while our iGDM is designed to maintain consistent measurements during the solving process. Therefore, based on their purposes, there are two major differences between our iGDM and Adam: (1) The Adam optimizer assumes that the parameters in the neural network follow the $L2$ distribution, so it applies a weight decay algorithm to preserve this assumption. In contrast, our iGDM does not make any prior assumption on the variables, because the prior is approximated by the diffusion models. Thus, iGDM does not include the weight decay algorithm. (2) The second momentum in the Adam optimizer is often unstable and requires a term added to the denominator to improve numerical stability. The value of this term would greatly affect the performance of the neural network. However, this term in iGDM has little influence on the performance (shown in Tab.~\ref{tab:function_optim}), which may be because our algorithm can provide a stable second moment. This term is necessary to prevent the unlikely case where the second momentum is zero in iGDM.

\begin{figure*}[t]
	\centering
	\includegraphics[width=1\linewidth]{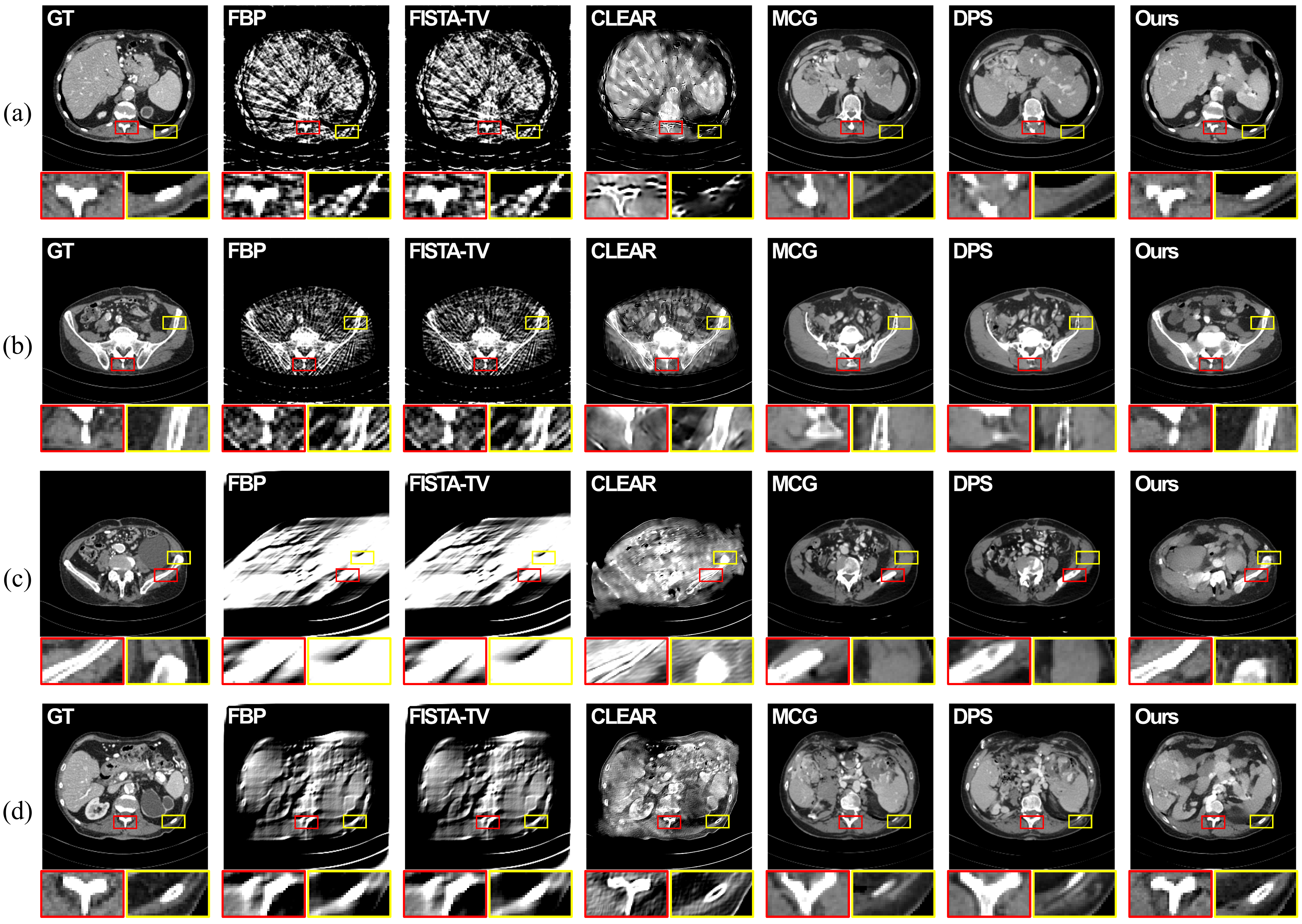}
	\caption{Qualitative results of medical sparse data reconstruction. (a) represents the visual results of 18 sparse-view CT reconstruction. (b) represents the visual results of 32 sparse-view CT reconstruction. (c) represents the visual results of 45 degree limited-angle CT reconstruction. (d) represents the visual results of 90 degree limited-angle CT reconstruction. The resolution of CT images is $256\times256$ The display window of CT images is set to $\left[-150, 256\right]$ HU. The standard measurement of CT is 512 views around 180 degrees.}
	\label{fig:mayo256vis}
\end{figure*}

\subsection{Theoretical analysis of history gradient update}
\label{sec:hgu}
In the following, we show the theoretical analysis of the gradient decent algorithm and our history gradient update in the context of diffusion-based solvers. The basic formulation of gradient guidance in Eq.~\ref{eq:lds} corresponds to a simple gradient descent scheme. However, previous work~\cite{chung2022improving,chung2022diffusion} did not provide convergence proofs for the use of gradient descent algorithms on the data fidelity term. Here, we give the proofs for the use of gradient descent algorithms on the data fidelity term. We also demonstrate through theoretical analysis that incorporating historical gradients into updates of the data fidelity term can still convergence.
The following definitions shows our definition on the data fidelity term and the convergence criteria:
\begin{definition}[Data fidelity term]
	For all $t \in \Rd$, the data fidelity term $U_t(\z)$ is a convex function related to latent variable $\z$, then any $\z_1$ and $\z_2$ in a given domain and $\forall a \in (0, 1)$ have:
	\begin{align}
		U_t(a\z_1 + (1-a)\z_2) & \leq a U_t(\z_1) + (1 - a)U_t(\z_2), \\
		U_t(\z_2) & \geq U_t(\z_1) + \left\langle \nabla U_t(\z_1),\z_2-\z_1\right\rangle .
	\end{align}
\end{definition}
\begin{definition}[Convergence Criteria]
	When $U_t(\z)$ is a convex function, the regret algorithm $R(T)$ from~\cite{Zinkevich2003OnlineCP} is chosen as the statistical quantity:
	\begin{equation}
		R(T) = \sum_{t=1}^T U_t \left(\z^{(t)}\right) - \min_\z \sum_{t=1}^T U_t (\z)	.
	\end{equation}
	When $T \to \infty$ and $R(T) / T \to 0$, we can conduct that the gradient descent algorithm is convergent, i.e., $\z \to \argmin_\z \sum_{t=1}^{T}U_t(\z) \triangleq \z^*$, not only does it converge to some value, but this value minimizes the data fidelity term. Subject to convergence of the algorithm, we generally assume that \textbf{1.} the slower $R(T)$ grows with $T$, the faster the algorithm converges, and \textbf{2.} the slower the guidance rate $\epsilon$ decays for the same growth rate, the faster the algorithm converges.
\end{definition}
Based on the above definitions, we make assumptions on the latent variable $\z$ and the gradient $\g$:
\begin{assumption}\label{var_bounded}
	Latent variable $\z$ is bounded, i.e., $\|\z - \z'\|_2 \leq D, \forall \z, \z'$, and the	gradient $\g = \nabla_{\z_t} \gU \left( \y, \Ac \mathcal{D}\left(\Ed \left[\z_0|\z_t\right]\right)\right) $ is also bounded, i.e., $\|\g_t\|_2 \leq G, \forall t$.
\end{assumption}
This assumption is very common and has been widely adopted in neural network analysis. Now we can provide a theoretical proof of convergence for the GD algorithm.
\begin{theorem}
	For the gradient descent algorithm used in~\cite{chung2022improving,chung2022diffusion}, we have the upper bound of GD as:
	\begin{equation}
		R(T) = \frac{1}{2\eps_T} D^2 + \frac{G^2}{2} \sum_{t=1}^{T} \eps_t^2.
	\end{equation}
	When $T \to \infty$, the GD algorithm is tend to convergence:
	\begin{equation}
		\lim_{T \to \infty} \frac{R(T)}{T} \leq \lim_{T \to \infty} \frac{1}{T} \left[\frac{1}{2\eps_T} D^2 + \frac{G^2}{2} \sum_{t=1}^{T} \eps_t^2\right] = 0.
	\end{equation}
	Considering $\eps_t$ is a function w.r.t $t$: $\eps_t = \epsilon(t)$, and a polynomial decay with a constant $C$ is employed: $\eps_t = C/t^n, n \geq 0$:
	\begin{equation}
		R(T) \leq D^2 \frac{T^n}{2 C} + \frac{G^2C}{2}\left(\frac{1}{1-n}T^{1-n} - \frac{n}{1-n}\right).
	\end{equation}
	Thus, $R(T) = \gO\left(T^{\max(n, 1-n)}\right)$. When $n=1/2$, $R(T)$ attain the optimal upper bound $\gO\left(T^{1/2}\right)$.
	\label{proofofgd}
\end{theorem}
By leveraging the result of Theorem~\ref{proofofgd}, we can ensure the GD algorithm can converge on the data fidelity term. In the training of neural networks, the integration of historical gradient information with gradient descent algorithms has achieved significant success~\cite{sutskever2013importance,kingma2014adam}. Therefore, similarly, we introduce historical gradient information into the optimization process for the data fidelity term to form our history gradient update method. Here, we demonstrate a variants of gradient update policies based on first and second moments. For any dimension $i$ of the variable $\z$, we have:
\begin{align}
	z^{(t+1)}_i & = z_i^{(t)} - \eps_t \frac{\hat{m}_i^{(t)}}{\sqrt{\hat{v}_i^{(t)}}} \nonumber \\
	             & = z_i^{(t)} - \frac{\eps_t}{1 - \prod_{s=1}^t \eta_{1,s}} \frac{\eta_{1,t} m^{(t-1)}_i + (1-\eta_{1,t})g_{t,i}}{\sqrt{\hat{v}_i^{(t)}}}, \nonumber \\
	             \hat{m}_i^{(t)} & =\frac{1-\eta_1}{1-\eta_1^t}\sum_{s=1}^{t}\eta_{1}^{t-s}g_s, \quad \hat{v}_i^{(t)}=\frac{1-\eta_2}{1-\eta_2^t}\sum_{s=1}^{t}\eta_{2}^{t-s}g_s^2. 
\end{align}
Where $\hat{m}_i^{(t)}, \hat{v}_i^{(t)}$ are the first and second momentum of gradient. $(\eta_1, \eta_2)$ are the coefficient for $\hat{m}_i^{(t)}, \hat{v}_i^{(t)}$. Based on the Assumption~\ref{var_bounded}, we have a theoretical proof convergence for Improved-Momentum-variant history gradient update:
\begin{theorem}
	For the Improved-Momentum-variant history gradient update, we have the upper bound as:
	\begin{align}
		& R(T) \leq \frac{\sum_{i=1}^{d}D^2_i G_i}{2\eps_T(1-\eta_{1,1})} + \left(\sum_{i=1}^{d}D_iG_i\right)\left(\sum_{t=1}^{T}\frac{\eta_{1,t}}{1-\eta_{1,t}}\right) + \nonumber \\
		& \left(\sum_{i=1}^{d}G_i\right)\left[\sum_{t=1}^{T}\frac{\gamma_t}{2(1-\eta_{1,t})} \cdot \sum_{s=1}^{t}\frac{(1-\eta_{1,s})^2\left(\prod_{r=s+1}^t\eta_{1,r}\right)^2}{(1-\eta_2)\eta_2^{t-s}}\right],
	\end{align}
	where $\gamma_t = \frac{\eps_t}{1 - \prod_{s=1}^t \eta_{1,s}}$. Similarly, 	When $T \to \infty$, this Improved-Momentum-variant history gradient update algorithm is tend to convergence. Considering $\eta_{1,t} \in (0, 1), \forall t, \eta_{1,1} \geq \eta_{1,2} \geq \dots \geq \eta_{1,T} \geq \dots$ and $\eta_2 \in (0,1), \frac{\eta_{1,t}}{\sqrt{\eta_2}}\leq\sqrt{c} < 1$, we have:
	\begin{align}
		& R(T) \leq \frac{\sum_{i=1}^{d}D^2_i G_i}{2\eps_T(1-\eta_{1,1})} + \left(\sum_{i=1}^{d}D_iG_i\right)\left(\sum_{t=1}^{T}\eta_{1,t}\right) + \nonumber \\
		& \left(\sum_{i=1}^{d}G_i\right)\left[\frac{\sum_{t=1}^{T}\eps_t}{2(1-\eta_{1,1})^2(1-\eta_2)(1-c)}\right].
	\end{align}
	Similar to Theorem~\ref{proofofgd}, when $n=1/2$, $R(T)$ attain the optimal upper bound $\gO\left(T^{1/2}\right)$.
	\label{proofofadam}
\end{theorem}
\begin{remark}
Note that the GD algorithm and our Improved-Momentum-variant history gradient update algorithm have similar upper bound which mean they can both converge when optimize the data fidelity term. Although Theorem~\ref{proofofgd}and~\ref{proofofadam} both require $T \to \infty$, in real-world applications, both the algorithms can converge with limited iterations. In addition, although both the algorithms have the same optimal upper bound $\gO\left(T^{1/2}\right)$, the Improved-Momentum-variant history gradient update has a faster and more stable convergence process than the naive GD algorithm when deals with the various inverse problems in real-world applications. While we have only proven the convergence of GD and iGDM, these proofs can also be extended to update algorithms using historical gradient information, such as Gradient Descent with Momentum (GDM)~\cite{sutskever2013importance}.
\end{remark}
It is worth noting that our history gradient update method is a generalize optimizing method for the data fidelity term, thus, we can apply it to our latent-diffusion-based solver~\cite{song2023solving2} and previous pixel-diffusion-based solvers (\textit{e.g.,\xspace} MCG~\cite{chung2022improving} and DPS~\cite{chung2022diffusion}). Details of the proof process are presented in the supporting document.

\begin{table*}[t]
	\centering
	\caption{Quantitative evaluation (PSNR, SSIM) of medical image reconstruction on AAPM test $256\times256$ dataset. We mark \textbf{bold} for the best and \underline{underline} for the second best. CLEAR~\cite{zhang2021clear} is a supervised method.}
		\begin{tabular}{@{\extracolsep{4pt}}lcccccccc@{}}
			\toprule
			\multicolumn{1}{l}{\multirow{3}{*}{\textbf{Method}}} & \multicolumn{4}{c}{\textbf{Sparse view}}                 & \multicolumn{4}{c}{\textbf{Limited angle }}   \\ \cmidrule{2-5} \cmidrule{6-9} & \multicolumn{2}{c}{18} & \multicolumn{2}{c}{32} & \multicolumn{2}{c}{45} & \multicolumn{2}{c}{90} \\ \cmidrule{2-3} \cmidrule{4-5} \cmidrule{6-7} \cmidrule{8-9} 
			\multicolumn{1}{c}{}  & PSNR $\uparrow$      & SSIM $\uparrow$     & PSNR $\uparrow$      & SSIM $\uparrow$     & PSNR $\uparrow$      & SSIM $\uparrow$     & PSNR $\uparrow$      & SSIM $\uparrow$  \\ \midrule
			FBP                                         &     24.76       &    0.5296       &     28.03       &     0.6779      &      16.65      &      0.5422     &      20.35      &     0.5113      \\
			FISTA-TV                                    &     24.86       &     0.5408      &      28.14      &      0.6888     &     16.66       &     0.5463      &    20.40        &    0.5241      \\
			CLEAR                                       &     \underline{32.28}      &     \underline{0.8798}     &      \underline{36.24}      &     \underline{0.9257}      &     25.71       &      \underline{0.8559}     &       \underline{31.60}     &    \textbf{0.9223}      \\
			MCG                                         &      28.54      &  0.8135         &     28.98       &     0.8242     &      26.08      &     0.7418      &     28.44      &     0.8079     \\
			DPS   & 28.55 & 0.8140 & 28.97 & 0.8242 & \underline{28.25} & 0.8204 & 28.25 & 0.8088 \\ \midrule
			LHGU  & \textbf{39.01} & \textbf{0.9552} & \textbf{39.77} & \textbf{0.9612} & \textbf{30.05} & \textbf{0.8747} & \textbf{32.68} & \underline{0.9032} \\ \bottomrule
		\end{tabular}%
	\label{tab:ct_images_256}
\end{table*}

\section{Experiments}

\subsection{Experimental setup}
\textbf{Models and datasets.} For medical image reconstruction, we train our DDPM and LDM model on the 2016 American Association of Physicists in Medicine (AAPM) grand challenge dataset~\cite{mccollough2017low}. The dataset has normal-dose data from 10 patients. 9 patients’ data are used for training, and 1 for validation which contains 526 images. To simulate low-dose imaging, a parallel-beam imaging geometry with 180 degrees was employed. Regarding inpainting and super-resolution tasks, we test our method on CelebAHQ 1k $256\times256$ dataset~\cite{liu2015faceattributes} and LSUN-bedroom $256\times256$ dataset~\cite{yu2015lsun}. We utilize pretrained DDPM and LDM models from the open-source model repository from~\cite{ho2020denoising,rombach2022high}. All the images are normalized to range $\left[0,1\right]$. More details including the hyper-parameters are listed in Appendix.~\hyperref[app:experiment]{B}. We call the Latent diffusion solver with HGU as LHGU.

\textbf{Measurement operators.} For sparse-view CT reconstruction, we uniformly sample 18 and 32 views. For limited-angle CT reconstruction, we restrict the imaging degree range to $45$ and $90$ degrees with 128 views using parallel beam geometry. For random inpainting, following~\cite{chung2022improving,chung2022diffusion}, we mask out $99\%$ of the total pixels (including all the channels). For super-resolution, we use $8\times$ bilinear downsampling. Gaussian noise is added in the nature image evaluation after a forward operation performed with $\sigma=0.05$. The medical data are evaluated without noise.

\begin{table*}[h]
	\centering
	\caption{Quantitative evaluation (PSNR, SSIM) of medical image reconstruction on AAPM test $512\times512$ dataset for zero-shot methods. We mark \textbf{bold} for the best and \underline{underline} for the second best.}
		\begin{tabular}{@{\extracolsep{4pt}}lcccccccc@{}}
			\toprule
			\multicolumn{1}{c}{\multirow{3}{*}{\textbf{Method}}} & \multicolumn{4}{c}{\textbf{Sparse view}}                 & \multicolumn{4}{c}{\textbf{Limited angle }}   \\ \cmidrule{2-5} \cmidrule{6-9} & \multicolumn{2}{c}{18} & \multicolumn{2}{c}{32} & \multicolumn{2}{c}{45} & \multicolumn{2}{c}{90} \\ \cmidrule{2-3} \cmidrule{4-5} \cmidrule{6-7} \cmidrule{8-9} 
			\multicolumn{1}{c}{}  & PSNR $\uparrow$       & SSIM $\uparrow$     & PSNR $\uparrow$      & SSIM $\uparrow$     & PSNR $\uparrow$      & SSIM $\uparrow$     & PSNR $\uparrow$      & SSIM $\uparrow$  \\ \midrule
			FBP                                         &      23.48      &     0.5096      &     26.70       &     0.6423      &     16.53       &     0.5480      &     19.88       &     0.4932      \\
			FISTA-TV                                    &     \underline{23.93}       &     \underline{0.5566}      &     \underline{27.11}       &     \underline{0.6768}      &      \underline{16.59}      &    \underline{0.5695}       &     \underline{20.08}       &     \underline{0.5348}     \\ \midrule
			LHGU                                        &     \textbf{36.94}       &    \textbf{0.9216}       &      \textbf{37.63}      &      \textbf{0.9373}     &      \textbf{30.47}      &     \textbf{0.8579}      &       \textbf{33.41}     &     \textbf{0.8837}     \\ \bottomrule
		\end{tabular}%
	\label{tab:ct_images_512}
\end{table*}
\subsection{Evaluation on medical data}

To assess the performance of HGU in reconstructing medical sparse data, we compare it with several recent state-of-the-art methods: manifold constrained gradients (MCG)~\cite{chung2022improving}, diffusion posterior sampling (DPS)~\cite{chung2022diffusion}, comprehensive learning enabled adversarial reconstruction (CLEAR)~\cite{zhang2021clear}, fast iterative shrinkage-thresholding algorithm with total-variation (FISTA-TV), and the analytical reconstruction method, filtered back projection (FBP). Peak-signal-to-noise-ratio (PSNR) and structural similarity index measure (SSIM) are used for quantitative evaluation. 

The quantitative results of medical sparse data reconstruction are demonstrated in Tab.~\ref{tab:ct_images_256} and Tab.~\ref{tab:ct_images_512}. LHGU outperforms all other state-of-the-art methods by a significant margin across all experiment settings. We also compare our method in the high-resolution CT image reconstruction task with zero-shot methods. However, due to the large memory consumption of DDPM, it is challenging to train DDPM models for high-resolution reconstruction. Thus, we exclude MCG and DPS which rely on DDPM from Tab.~\ref{tab:ct_images_512}. The results show that LHGU provides noise-free reconstruction results, although there is still a significant gap between the reconstructed images and the ground truth. In contrast, other zero-shot methods fail to reconstruct meaningful results. 


The qualitative results of medical sparse image reconstruction are demonstrated in Fig.~\ref{fig:mayo256vis} which are consistent with the quantitative results reported in Tab.~\ref{tab:ct_images_256}. In Fig.~\ref{fig:mayo256vis}, we compare our method with the state-of-the-art zero-shot unsupervised and supervised methods. We observe that LHGU can provide high-quality reconstructions, especially for the sparse view reconstruction task. Specifically, LHGU can provide better overall structure and nearly artifact-free reconstruction. Additionally, our method also provides better reconstructions than other methods in limited angle reconstruction tasks. 

\begin{figure*}[t]
	\centering
	\includegraphics[width=1\linewidth]{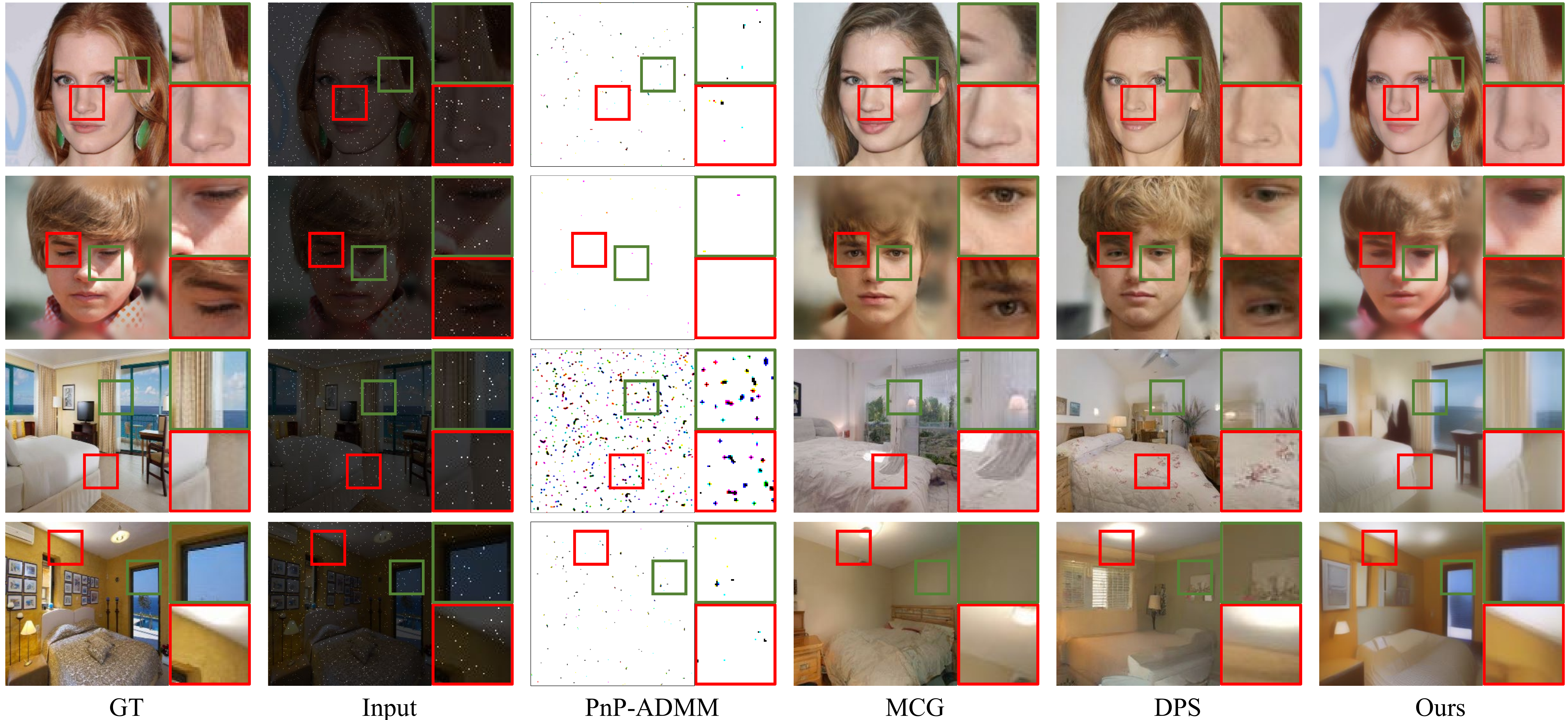}
	\caption{Qualitative results of $99\%$ inpaint nature image reconstruction with Gaussian noise ($\sigma=0.05$).}
	\label{fig:naturevis}
\end{figure*}
\begin{figure*}[t]
	\centering
	\includegraphics[width=1\linewidth]{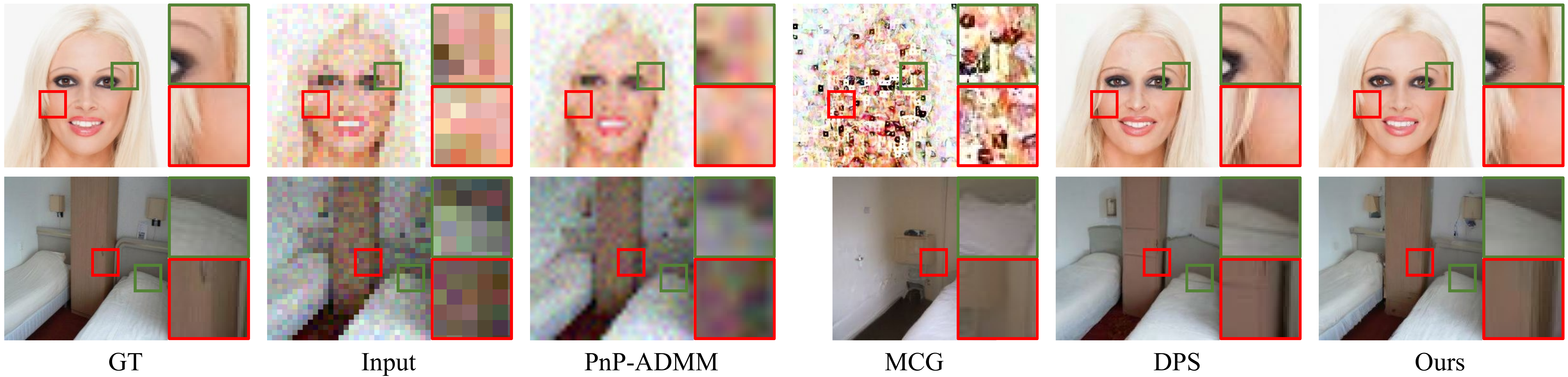}
	\caption{Qualitative results of $8\times$ super-resolution nature image reconstruction with Gaussian noise ($\sigma=0.05$).}
	\label{fig:naturevis_sr}
\end{figure*}

\begin{table*}[t]
	\centering
	\caption{Quantitative evaluation (PSNR, SSIM, LPIPS) of nature image reconstruction on CelebAHQ and LSUN-bedroom dataset. We mark \textbf{bold} for the best and \underline{underline} for the second best.}
	\resizebox{\textwidth}{!}{%
		\begin{tabular}{@{\extracolsep{1pt}}lccccccccccccc@{}}
			\toprule
			\multicolumn{1}{c}{\multirow{3}{*}{\textbf{Method}}}     & \multirow{3}{*}{\textbf{Type}}                                      & \multicolumn{6}{c}{\textbf{CelebAHQ}}                           & \multicolumn{6}{c}{\textbf{LSUN-bedroom}}                             \\ \cmidrule{3-8} \cmidrule{9-14}
			\multicolumn{2}{c}{}                                                                   & \multicolumn{3}{c}{\textbf{Inpaint}} & \multicolumn{3}{c}{\textbf{SR (8$\times$)}} & \multicolumn{3}{c}{\textbf{Inpaint}} & \multicolumn{3}{c}{\textbf{SR (8$\times$)}} \\ \cmidrule{3-5} \cmidrule{6-8} \cmidrule{9-11} \cmidrule{12-14}
			\multicolumn{2}{c}{}                                                                    & PSNR $\uparrow$    & SSIM $\uparrow$   & LPIPS $\downarrow$   & PSNR $\uparrow$  & SSIM $\uparrow$  & LPIPS $\downarrow$ & PSNR $\uparrow$    & SSIM $\uparrow$   & LPIPS $\downarrow$  & PSNR $\uparrow$ & SSIM $\uparrow$ & LPIPS $\downarrow$ \\ \midrule
			PnP-ADMM &       Traditional IR                                                 &    3.97     &     0.3017    &    0.8916     &   22.94    &   0.6303    &    0.6820    &    5.059     &    0.3236     &    0.8940     &    \textbf{20.14}   &   0.5458    &    0.7944    \\
			MCG      & Pixel Diffusion                                                          &    18.91     &     0.5600    &    0.2544     &    12.47   &   0.1655    &   0.6713     &     16.89    &    0.4555     &    0.5486     &    9.39   &    0.0606   &   0.8698     \\
			DPS          & Pixel Diffusion                                                      &    \underline{18.95}     &     \underline{0.5614}    &    \underline{0.2543 }    &   \underline{24.36}    &    \underline{0.7116}   &   \underline{0.1089}     &     \underline{17.03}    &    \underline{0.4587}     &    \underline{0.5414}     &   19.15    &    \underline{0.5614}   &    \textbf{0.3074}    \\ \midrule
			LHGU         & Latent Diffusion                                                      &    \textbf{22.14}     &     \textbf{0.6647 }   &     \textbf{0.2280}    &   \textbf{25.27}    &    \textbf{0.7530}   &     \textbf{0.0878}   &    \textbf{20.33}     &    \textbf{0.5845}     &   \textbf{0.4858}      &   \underline{19.83}    &    \textbf{0.5762}   &    \underline{0.3253}    \\
			\bottomrule
		\end{tabular}%
	}
	\label{tab:nature_images}
\end{table*}

\subsection{Evaluation on nature images}
In order to further test the performance of our method, we compare our method against state-of-the-art methods, namely, MCG, DPS, and plug-and-play alternating direction method of multipliers (PnP-ADMM)~\cite{chan2016plug}. For quantitative analysis, we utilize three widely used perceptual evaluation metrics: LPIPS distance, PSNR, and SSIM.

The quantitative results of nature image reconstruction are illustrated in Tab.~\ref{tab:nature_images}. Our method achieves competitive results compared to the previous state-of-the-art. Specifically, we observe that our method is able to accurately reconstruct the original data and preserve the most data consistency, even when dealing with highly sparse measurements such as $99\%$ random inpainting. Additionally, we note that LHGU gains some advantages over the previous best method on the super-resolution task. 

The qualitative results of nature sparse image reconstruction are demonstrated in Fig.~\ref{fig:naturevis} and Fig.~\ref{fig:naturevis_sr}. Notably, the traditional iterative method PnP-ADMM failed to produce satisfactory results for both the inpainting and super-resolution tasks due to its limited prior terms. In contrast, our method outperforms the comparison methods, particularly in terms of color and structure in the inpainting task. In the super-resolution task, the results obtained by MCG exhibit many artifacts, which are likely due to the projection step\cite{chung2022diffusion}. Our method, on the other hand, achieves competitive results with DPS, the most advanced method, with small gaps.

\begin{table*}[]
\centering
\caption{Ablation evaluation (PSNR, SSIM) on the effect of our method. HGU denotes the Improved-Momentum-variant HGU. Model A is equal to DPS~\cite{chung2022diffusion} We mark \textbf{bold} for the best and \underline{underline} for the second best.}
\label{tab:ablation_dir}
\begin{tabular}{@{\extracolsep{1pt}}lcccccccccccc@{}}
\toprule
\multirow{3}{*}{\textbf{Model}} &
  \multirow{3}{*}{\textbf{\begin{tabular}[c]{@{}c@{}}Diffusion\\ Type\end{tabular}}} &
  \multirow{3}{*}{\textbf{\begin{tabular}[c]{@{}c@{}}Update\\ Strategy\end{tabular}}} &
  \multicolumn{4}{c}{\textbf{Sparse View}} &
  \multicolumn{4}{c}{\textbf{Limited Angle}} &
  \multirow{3}{*}{\textbf{\begin{tabular}[c]{@{}c@{}}Speed\\ (iter/s)\end{tabular}}} &
  \multirow{3}{*}{\textbf{\begin{tabular}[c]{@{}c@{}}Memory\\ (MB)\end{tabular}}} \\ \cmidrule(lr){4-7} \cmidrule(lr){8-11}  
 &
   &
   &
  \multicolumn{2}{c}{18} &
  \multicolumn{2}{c}{32} &
  \multicolumn{2}{c}{45} &
  \multicolumn{2}{c}{90} &
   &
   \\ \cmidrule(lr){4-5} \cmidrule(lr){6-7} \cmidrule(lr){8-9} \cmidrule(lr){10-11}   
 &
   &
   &
  PSNR &
  SSIM &
  PSNR &
  SSIM &
  PSNR &
  SSIM &
  PSNR &
  SSIM &
   &
   \\ \midrule
A (DPS) &
  Pixel &
  GD &
  28.55 &
  0.8140 &
  28.97 &
  0.8242 &
  28.25 &
  0.8204 &
  28.25 &
  0.8088 &
  20.88 &
  {\ul 6338} \\
B &
  Latent &
  GD &
  31.37 &
  0.8775 &
  31.88 &
  0.8891 &
  \textbf{29.82} &
  {\ul 0.8684} &
  31.65 &
  0.8884 &
  \textbf{37.03} &
  \textbf{4268} \\
C &
  Pixel &
  HGU &
  {\ul 32.83} &
  {\ul 0.8892} &
  {\ul 34.29} &
  {\ul 0.9171} &
  29.40 &
  0.8562 &
  {\ul 32.42} &
  {\ul 0.9047} &
  20.62 &
  {\ul 6338} \\
D (Ours, LHGU) &
  Latent &
  HGU &
  \textbf{39.01} &
  \textbf{0.9552} &
  \textbf{39.77} &
  \textbf{0.9612} &
  {\ul 29.60} &
  \textbf{0.8779} &
  \textbf{32.89} &
  \textbf{0.9116} &
  {\ul 36.67} &
  \textbf{4268} \\ \bottomrule
\end{tabular}
\end{table*}

\subsection{Ablation studies}
\label{manu:ablation}
We conducted ablation studies to validate the effectiveness of our approach. we compared the performance of our approach against a pixel-based iterative reconstruction approach. To ensure a fair comparison, we conducted these ablation studies on the medical image reconstruction task, as both the DDPM and LDM models were trained using the same protocol.

In Table~\ref{tab:ablation_dir}, we can observe that our method outperforms the pixel-based iterative reconstruction method, DPS, by a large margin. This result confirms that the latent-based approach is superior to the pixel-based approach in terms of both speed and accuracy. Additionally, we can see that using HGU in the pixel space can improve the performance of DPS, allowing it to surpass the naive DPS. Compared to pixel-space models, LHGU achieves significant speed-up with less memory consumption. Although LHGU decodes latent into pictures at each step, it still has a greater advantage than processing directly in pixel space.

From Table~\ref{tab:function_optim}, we can observe that our history gradient update method significantly enhances performance, regardless of the evaluation function used, even for non-strong-convex functions such as LPIPS~\cite{johnson2016perceptual} and FFL~\cite{jiang2021focal}. This suggests that our history gradient update is a versatile method applicable to various scenarios. Similarly, from Fig.~\ref{fig:eval_func}, we can observe that in the case of using Improved-Momentum-variant, all evaluation functions achieve relatively stable optimization. In Fig.~\ref{fig:hgu_func}, we present a comparison between our history gradient update and the GD algorithm. It is evident that our algorithm provides a more stable optimization process and does not exhibit an increase in error during the later stages of optimization. This indicates that historical gradient information contributes to a more stable optimization of the data fidelity term.

\begin{figure*}[t]
	\begin{minipage}{0.5\textwidth}
		\centering
		\includegraphics[width=\linewidth]{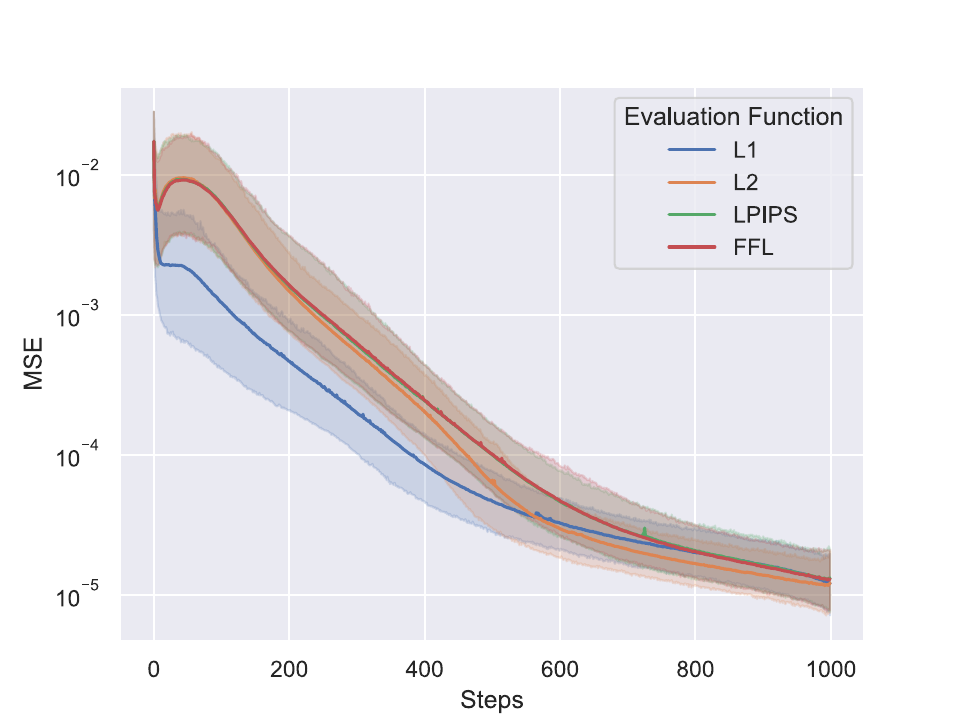}
		\caption{Errors of different evaluation functions using Improved-Momentum-variant LHGU in sparse-view CT reconstruction.}
		\label{fig:eval_func}
	\end{minipage}
	\begin{minipage}{0.5\textwidth}
		\centering
		\includegraphics[width=\linewidth]{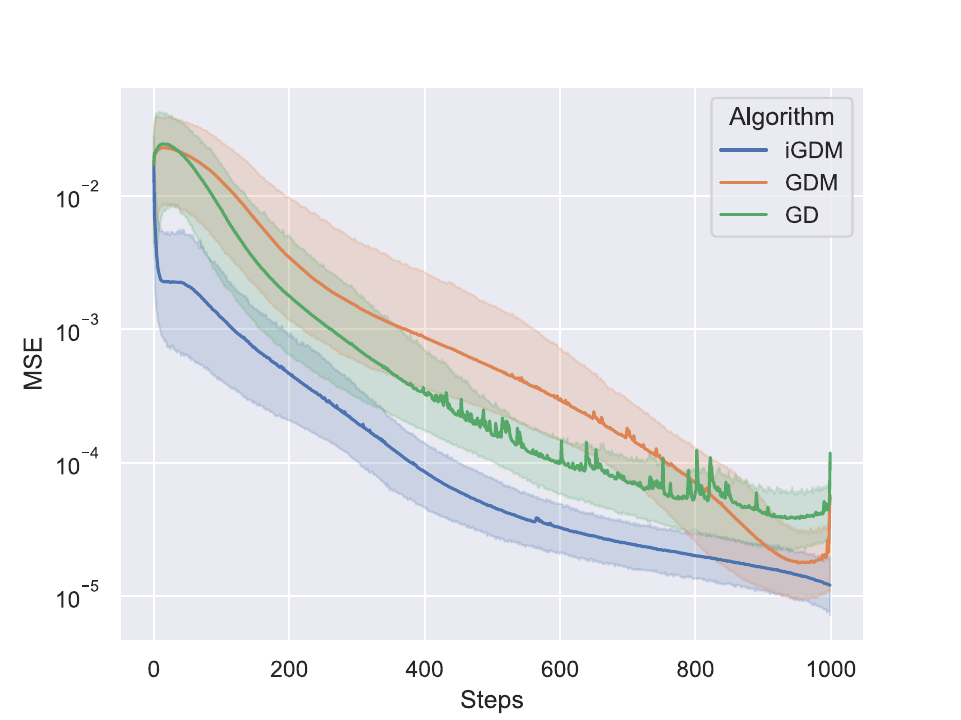}
		\caption{Errors of different gradient update algorithms in sparse-view CT reconstruction. The bands show $90\%$ percentile interval.}
		\label{fig:hgu_func}
	\end{minipage}
\end{figure*}

\begin{table*}[t]
	\centering
	\caption{Ablation evaluation (PSNR, SSIM) on the effect of evaluation function and history gradient update variants. We mark \textbf{bold} for the best and \underline{underline} for the second best.}
	\label{tab:function_optim}
		\begin{tabular}{@{}lcccccccccccc@{}}
			\toprule
			Evaluation function     & L1     & L1     & L1              & L2 & L2     & L2              & FFL    & FFL    & FFL          & LPIPS  & LPIPS  & LPIPS  \\
			\midrule
			Gradient update method  & GD     & GDM    & iGDM            & GD                & GDM    & iGDM            & GD     & GDM    & iGDM         & GD     & GDM    & iGDM   \\ \midrule
			PSNR $\uparrow$                   & 27.07  & 38.87  & {\ul 39.77}     & 32.23             & 32.36  & \textbf{39.86}  & 23.18  & 34.30  & 39.56        & 19.03  & 23.30  & 39.54  \\
			SSIM $\uparrow$                   & 0.8857 & 0.9551 & \textbf{0.9612} & 0.8933            & 0.8948 & \textbf{0.9612} & 0.7822 & 0.9009 & {\ul 0.9598} & 0.6144 & 0.6409 & 0.9597 \\ \bottomrule
		\end{tabular}%
\end{table*}


\begin{table*}[]
	\centering
	\caption{Hyper-parameters evaluation on the Improved-Momentum-variant history gradient update. We mark \textbf{bold} for the best.}
	\label{tab:adam}
	\begin{tabular}{@{}lccccccccc<{}}
		\toprule
		Model & Baseline & A      & B      & C      & D      & E      & F      & G      & H      \\ \midrule
		$\varepsilon$ &
		$1e^{-8}$ &
		\cellcolor[HTML]{EFEFEF}$1e^{-3}$ &
		\cellcolor[HTML]{EFEFEF}$1$ &
		$1e^{-8}$ &
		$1e^{-8}$ &
		$1e^{-8}$ &
		$1e^{-8}$ &
		$1e^{-8}$ &
		$1e^{-8}$ \\
		$\eta_1$ &
		$0.9$ &
		$0.9$ &
		$0.9$ &
		\cellcolor[HTML]{C0C0C0}$0.99$ &
		\cellcolor[HTML]{C0C0C0}$0.5$ &
		\cellcolor[HTML]{C0C0C0}$0.1$ &
		$0.9$ &
		$0.9$ &
		$0.9$ \\
		$\eta_2$ &
		$0.999$ &
		$0.999$ &
		$0.999$ &
		$0.999$ &
		$0.999$ &
		$0.999$ &
		\cellcolor[HTML]{9B9B9B}$0.9$ &
		\cellcolor[HTML]{9B9B9B}$0.5$ &
		\cellcolor[HTML]{9B9B9B}$0.1$ \\ \midrule
		PSNR  & \textbf{39.77}    & 39.74  & \textbf{39.77}  & 30.85  & 34.11  & 30.83  & 29.88  & 20.02  & 20.85  \\
		SSIM  & \textbf{0.9612 }  & 0.9609 & 0.9610 & 0.8465 & 0.9200 & 0.8720 & 0.8245 & 0.6252 & 0.6414 \\ \bottomrule
	\end{tabular}
\end{table*}
\subsection{Hyper-parameter studies}
To analyze the impact of the history gradient update on the entire reconstruction process, we conducted multiple experiments on the hyperparameters of its Improved-Momentum-variant. Table~\ref{tab:adam} presents the experimental results of adjusting the three hyperparameters, $\varepsilon, \eta_1, \eta_2$. From the results, it is evident that adjusting $\varepsilon$ does not significantly impact the final performance. However, modifications to $\eta_1$ and $\eta_2$ result in a substantial reduction in performance, a phenomenon reminiscent of using the Adam optimizer for neural network optimization. On the other hand, this phenomenon underscores the necessity of incorporating historical gradients into the optimization process. Proper hyperparameter selection can lead to a significant performance improvement.

\section{Conclusion}

In this paper, we introduce the History Gradient Update (HGU) as a powerful optimization tool for solving general inverse problems. We ensured the convergence of optimizing the data fidelity term using a gradient descent algorithm through theoretical derivation. Simultaneously, we demonstrated that incorporating historical gradient information into the optimization process accelerates its convergence. In addition, we generate the prior term in the latent space instead of the pixel space, which has proven its effectiveness on small-scale datasets and also leads to faster sampling speed. Our experimental results demonstrate that the latent diffusion solver with HGU outperforms state-of-the-art methods including the supervised method on sparse CT data reconstruction and achieves competitive results on nature image restoration. We believe that our work offers the community a promising tool for leveraging the rapidly growing field of latent diffusion models to restore high-quality and high-resolution data from degraded measurements with stable and fast optimizations.

\section*{Acknowledgments}
This work were supported in part by the Sichuan Science and Technology Program under Grant 2022JDJQO045 and Grant 2021JDJQ0024, in part by the National Natural Science Foundation of China under Grant 61871277 and Grant 62271335, in part by the Sichuan Health Commission Research Project under Grant 19PJ007, in part by the Chengdu Municipal Health Commission Research Project under Grant 2022053, in part by the Chengdu Key Research and development Support project under Grant 2021YF0501788SN.



%
\bibliographystyle{IEEEtran}
\bibliography{tmm}


%
%
%
%
%
%
%
%

\clearpage

\appendices

\section*{Proofs}
\setcounter{theorem}{1}
\begin{theorem}
	For the gradient descent algorithm used in~\cite{chung2022improving,chung2022diffusion}, we have the upper bound of GD as:
	\setcounter{equation}{13}
	\begin{equation}
		R(T) = \frac{1}{2\eps_T} D^2 + \frac{G^2}{2} \sum_{t=1}^{T} \eps_t^2.
	\end{equation}
	When $T \to \infty$, the GD algorithm is tend to convergence:
	\begin{equation}
		\lim_{T \to \infty} \frac{R(T)}{T} \leq \lim_{T \to \infty} \frac{1}{T} \left[\frac{1}{2\eps_T} D^2 + \frac{G^2}{2} \sum_{t=1}^{T} \eps_t^2\right] = 0.
	\end{equation}
	Considering $\eps_t$ is a function w.r.t $t$: $\eps_t = \epsilon(t)$, and a polynomial decay with a constant $C$ is employed: $\eps_t = C/t^n, n \geq 0$:
	\begin{equation}
		R(T) \leq D^2 \frac{T^n}{2 C} + \frac{G^2C}{2}\left(\frac{1}{1-n}T^{1-n} - \frac{n}{1-n}\right).
	\end{equation}
	Thus, $R(T) = \gO\left(T^{\max(n, 1-n)}\right)$. When $n=1/2$, $R(T)$ attain the optimal upper bound $\gO\left(T^{1/2}\right)$.
\end{theorem}
\begin{proof}
	\setcounter{equation}{26}
	Considering $\z^* = \argmin_\z \sum_{t=1}^{T} U_t (\z)$, we have:
	\begin{align}
		R(T) &= \sum_{t=1}^T U_t \left(\z^{(t)}\right) - \min_\z \sum_{t=1}^T U_t (\z) \\
		     &= \sum_{t=1}^T U_t \left(\z^{(t)}\right) - \sum_{t=1}^T U_t (\z^*) \\
		     &= \sum_{t=1}^T \left[\left(\z^{(t)}\right) - U_t (\z^*) \right],
	\end{align}
	because $U_t(\z)$ is a convex function, we have:
	\begin{align}
		& U_t (\z^*) \geq U_t \left(\z^{(t)}\right) + \left\langle \g_t, \z^{*} - \z^{(t)}\right\rangle \\
		\Longrightarrow & U_t \left(\z^{(t)}\right) - U_t (\z^*) \leq \left\langle \g_t,\z^{(t)} - \z^{*}\right\rangle.
	\end{align}
	Hence, we can substitute the above expression into $R(T)$:
	\begin{equation}
		R(T) \leq \sum_{t=1}^{T} \left\langle \g_t,\z^{(t)} - \z^{*}\right\rangle.
	\end{equation}
	For the GD algorithm, we have:
	\begin{align}
						\z^{(t+1)}  = & \z^{(t)} - \eps_t \g_t \\
		\Longrightarrow \z^{(t+1)} - \z^*  = & \z^{(t)} - \z^* - \eps_t \g_t \\
		\Longrightarrow \left\|\z^{(t+1)} - \z^*\right\|_2^2  = & \left\|\z^{(t)} - \z^* - \eps_t \g_t\right\|_2^2 \\
		\Longrightarrow \left\|\z^{(t+1)} - \z^*\right\|_2^2  = & \left\|\z^{(t)} - \z^*\right\|_2^2 - 2 \eps_t \left\langle \g_t,\z^{(t)} - \z^{*}\right\rangle \nonumber \\
		& + \eps_t^2 \|\g_t\|_2^2 \\
		\Longrightarrow \left\langle \g_t,\z^{(t)} - \z^{*}\right\rangle = & \frac{1}{2\eps_t}\left[\left\|\z^{(t)} - \z^*\right\|_2^2 - \left\|\z^{(t+1)} - \z^*\right\|_2^2 \right] \nonumber \\
		& + \frac{\eps_t}{2} \|\g_t\|_2^2.
	\end{align}
	Thus, the upper bound of $R(T)$ can be:
	\begin{align}
		R(T) & \leq \sum_{i=1}^T \frac{1}{2\eps_t}\left[\left\|\z^{(t)} - \z^*\right\|_2^2 - \left\|\z^{(t+1)} - \z^*\right\|_2^2 \right] + \frac{\eps_t}{2} \|\g_t\|_2^2 \\
			&=\underbrace{\sum_{i=1}^T \frac{1}{2\eps_t}\left[\left\|\z^{(t)} - \z^*\right\|_2^2 - \left\|\z^{(t+1)} - \z^*\right\|_2^2 \right]}_{(1)} \nonumber \\
			&+\underbrace{\sum_{i=1}^T \frac{\eps_t}{2} \|\g_t\|_2^2}_{(2)}. \label{eq:upper_bound_group}
	\end{align}
	For the $(1)$ in Eq.~\ref{eq:upper_bound_group}, we have:
	\begin{align}
		& \sum_{i=1}^T \frac{1}{2\eps_t}\left[\left\|\z^{(t)} - \z^*\right\|_2^2 - \left\|\z^{(t+1)} - \z^*\right\|_2^2 \right] \\
	   =& \frac{1}{2\eps_1}\left\| \z^{(1)} - \z^* \right\|_2^2 - \frac{1}{2\eps_1}\left\| \z^{(2)} - \z^* \right\|_2^2 + \\
	    & \frac{1}{2\eps_2}\left\| \z^{(2)} - \z^* \right\|_2^2 - \frac{1}{2\eps_2}\left\| \z^{(3)} - \z^* \right\|_2^2 + \cdots + \\
	    & \frac{1}{2\eps_T}\left\| \z^{(T)} - \z^* \right\|_2^2 - \frac{1}{2\eps_T}\left\| \z^{(T+1)} - \z^* \right\|_2^2 \\
	   =& \frac{1}{2\eps_1}\left\| \z^{(1)} - \z^* \right\|_2^2 + \sum_{t=2}^T\left(\frac{1}{2\eps_t} - \frac{1}{2\eps_{t-1}}\right)\left\|\z^{(t)} - \z^*\right\|_2^2 \\
	   &- \frac{1}{2\eps_T}\left\| \z^{(T+1)} - \z^* \right\|_2^2.
	\end{align}
	Since the latent variable $\z$ is bounded, we have:
	\begin{equation}
		\frac{1}{2\eps_1} \left\| \z^{(1)} - \z^* \right\|_2^2 \leq \frac{1}{2\eps_1} D^2.
	\end{equation}
	Since $\left\lbrace \eps_t\right\rbrace $ is monotonically non-decreasing and the latent variable $\z$ is bounded, we have:
	\begin{equation}
		\sum_{t=2}^T \left(\frac{1}{2\eps_t} - \frac{1}{2\eps_{t-1}}\right) \left\|\z^{(t)} - \z^*\right\|_2^2 \leq \sum_{t=2}^T \left(\frac{1}{2\eps_t} - \frac{1}{2\eps_{t-1}}\right) D^2.
	\end{equation}
	Clearly, $- \frac{1}{2\eps_T}\left\| \z^{(T+1)} - \z^* \right\|_2^2$ is less than or equal to 0. Therefore, we can rescale equation $(1)$ as:
	\begin{align}
		& \sum_{i=1}^T \frac{1}{2\eps_t}\left[\left\|\z^{(t)} - \z^*\right\|_2^2 - \left\|\z^{(t+1)} - \z^*\right\|_2^2 \right] \\
   \leq & \frac{1}{2\eps_1} D^2 + \sum_{t=2}^T \left(\frac{1}{2\eps_t} - \frac{1}{2\eps_{t-1}}\right) D^2 + 0 = \frac{1}{2\eps_T}D^2.
	\end{align}
	And for $(2)$, considering the gradient $\g$ is bounded, we have:
	\begin{equation}
		\sum_{i=1}^T \frac{\eps_t}{2} \|\g_t\|_2^2 \leq \sum_{i=1}^T \frac{\eps_t}{2} G^2 = \frac{G^2}{2} \sum_{t=1}^T \eps_t^2.
	\end{equation}
	Thus, we have the upper bound for $R(T)$:
	\begin{equation}
		R(T) = \frac{1}{2\eps_T} D^2 + \frac{G^2}{2} \sum_{t=1}^{T} \eps_t^2.
	\end{equation}
	Clearly, when $T \to \infty$, the above equation is tend to $0$:
	\begin{equation}
		\lim_{T \to \infty} \frac{R(T)}{T} \leq \lim_{T \to \infty} \frac{1}{T} \left[\frac{1}{2\eps_T} D^2 + \frac{G^2}{2} \sum_{t=1}^{T} \eps_t^2\right] = 0.
	\end{equation}
	Thus, we can conclude that the GD algorithm is tend to convergence when $T \to \infty$. Considering $\eps_t$ is a function w.r.t $t$: $\eps_t = \epsilon(t)$, and a polynomial decay with a constant $C$ is employed: $\eps_t = C/t^n, n \geq 0$:
	\begin{align}
		R(T) &\leq \frac{T^p}{2C} D^2 + \frac{G^2}{2} \sum_{t=1}^{T} \frac{C}{t^p} \\
		     &\leq \frac{T^p}{2C} D^2 + \frac{G^2C}{2} \left(1 + \int_{1}^{T}\frac{dt}{t^p}\right) \\
		     &= \frac{T^p}{2C} D^2 + \frac{G^2C}{2} \left(\frac{1}{1-p}T^{1-p} - \frac{p}{1-p}\right).
	\end{align}
	Thus, $R(T) = \gO\left(T^{\max(n, 1-n)}\right)$. When $n=1/2$, $R(T)$ attain the optimal upper bound $\gO\left(T^{1/2}\right)$.	
\end{proof}

\begin{theorem}
	\setcounter{equation}{17}
	For the Improved-Momentum-variant history gradient update, we have the upper bound as:
	\begin{align}
		& R(T) \leq \frac{\sum_{i=1}^{d}D^2_i G_i}{2\eps_T(1-\eta_{1,1})} + \left(\sum_{i=1}^{d}D_iG_i\right)\left(\sum_{t=1}^{T}\frac{\eta_{1,t}}{1-\eta_{1,t}}\right) + \nonumber \\
		& \left(\sum_{i=1}^{d}G_i\right)\left[\sum_{t=1}^{T}\frac{\gamma_t}{2(1-\eta_{1,t})} \cdot \sum_{s=1}^{t}\frac{(1-\eta_{1,s})^2\left(\prod_{r=s+1}^t\eta_{1,r}\right)^2}{(1-\eta_2)\eta_2^{t-s}}\right],
	\end{align}
	where $\gamma_t = \frac{\eps_t}{1 - \prod_{s=1}^t \eta_{1,s}}$. Similarly, 	When $T \to \infty$, this Improved-Momentum-variant history gradient update algorithm is tend to convergence. Considering $\eta_{1,t} \in (0, 1), \forall t, \eta_{1,1} \geq \eta_{1,2} \geq \dots \geq \eta_{1,T} \geq \dots$ and $\eta_2 \in (0,1), \frac{\eta_{1,t}}{\sqrt{\eta_2}}\leq\sqrt{c} < 1$, we have:
	\begin{align}
		& R(T) \leq \frac{\sum_{i=1}^{d}D^2_i G_i}{2\eps_T(1-\eta_{1,1})} + \left(\sum_{i=1}^{d}D_iG_i\right)\left(\sum_{t=1}^{T}\eta_{1,t}\right) + \nonumber \\
		& \left(\sum_{i=1}^{d}G_i\right)\left[\frac{\sum_{t=1}^{T}\eps_t}{2(1-\eta_{1,1})^2(1-\eta_2)(1-c)}\right].
	\end{align}
	Similar to Theorem~\ref{proofofgd}, when $n=1/2$, $R(T)$ attain the optimal upper bound $\gO\left(T^{1/2}\right)$.
\end{theorem}
\begin{proof}
	\setcounter{equation}{55}
	For the Improved-Momentum-variant algorithm, we have:
	\begin{align}
		\m^{(t)} &= \eta_1 \m^{(t-1)} + (1-\eta_1)\g_t, \hat{\m}^{(t)} = \frac{\m^{(t)}}{1-\eta_1^t}, \m^{(0)} = 0, \\
		\vbb^{(t)} &= \eta_2 \vbb^{(t-1)} + (1-\eta_2)\g_t^2, \hat{\vbb}^{(t)} = \frac{\vbb^{(t)}}{1-\eta_2^t},\vbb^{(0)}=0, \\
		\z^{(t+1)} &= \z^{(t)} - \eps_t \frac{\hat{\m}^{(t)}}{\sqrt{\hat{\vbb}^{(t)}}}, \; \g_t = \nabla U_t(\z^{(t)}).
	\end{align}
	We can calculate its closed-form solution based on the iterative equation for $\m^{(t)}$ as:
	\begin{align}
		\m^{(t)} =& \eta_1\m^{(t-1)} + (1-\eta_1)\g_t \\
		         =& \eta_1^2\m^{(t-2)} + \eta_1(1-\eta_1)\g_{t-1} + (1-\eta_1)\g_t \\
		         =& \eta_1^3\m^{(t-3)} + \eta_1^2(1-\eta_1)\g_{t-2} + \eta_1(1-\eta_1)\g_{t-1} + \nonumber \\ & (1-\eta_1)\g_t \\
		         =& \eta_1^t\m^{(0)} + \eta_1^{t-1}(1-\eta_1)\g_1 + \cdots + \eta_1^2(1-\eta_1)\g_{t-2} + \nonumber \\ & \eta_1(1-\eta_1)\g_{t-1} + (1-\eta_1)\g_t \\
		         \overset{\text{(a)}}{=} & \eta_1^{t-1}(1-\eta_1)\g_1 + \cdots + \eta_1^2(1-\eta_1)\g_{t-2} + \nonumber \\ & \eta_1(1-\eta_1)\g_{t-1} + (1-\eta_1)\g_t \\
		         =& (1-\eta_1)\sum_{s=1}^{t}\eta_1^{t-s}\g_s.
	\end{align}
	Where (a) holds because $\m^{(0)}=0$. Similarly, we have:
	\begin{equation}
		\vbb^{(t)} = (1-\eta_2)\sum_{s=1}^{t}\eta_2^{t-s}\g_s^2.
	\end{equation}
	For $\Ed\left[\m^{(t)}\right]$ and $\Ed\left[\vbb^{(t)}\right]$, we have:
	\begin{align}
		& \Ed\left[\m^{(t)}\right] = (1-\eta_1)\sum_{s=1}^{t}\eta_1^{t-s} \Ed \left[\g_s\right] = \Ed \left[\g_s\right] \cdot (1-\eta_1^t), \\ 
		&\Ed\left[\vbb^{(t)}\right] = (1-\eta_2)\sum_{s=1}^{t}\eta_2^{t-s} \Ed \left[\g_s^2\right] = \Ed \left[\g_s^2\right] \cdot (1-\eta_2^t).
	\end{align}
	In order to make $\Ed\left[\m^{(t)}\right]=\Ed \left[\g_s\right]$ and $\Ed\left[\vbb^{(t)}\right] = \Ed \left[\g_s^2\right]$, we need to make the following corrections:
	\begin{align}
		\m^{(t)} &\rightarrow \hat{\m}^{(t)} = \frac{\m^{(t)}}{1-\eta_1^t}, \\
		\vbb^{(t)} &\rightarrow  \hat{\vbb}^{(t)} = \frac{\vbb^{(t)}}{1-\eta_2^t}.
	\end{align}
	We break down the upper bound of $R(T)$ into individual dimensions of the variables as:
	\begin{align}
		R(T) \leq& \sum_{t=1}^{T} \left\langle \g_t,\z^{(t)} - \z^{*}\right\rangle \\
			 	=& \sum_{t=1}^{T} \sum_{i=1}^{d} g_{t,i}\left(z_i^{(t)} - z_i^*\right) \\
			 	=& \sum_{i=1}^{d} \sum_{t=1}^{T} g_{t,i}\left(z_i^{(t)} - z_i^*\right).
	\end{align}
	For any dimension $i$ of the variable, we have:
	\begin{align}
		z_i^{(t+1)} &= z_i^{(t+1)} - \eps_t \frac{\hat{m}^{(t)}_i}{\sqrt{\hat{v}^{(t)}_i}} \\
					&= z_i^{(t+1)} - \eps_t \frac{1}{1-\eta_1^t}\frac{m^{(t)}_i}{\sqrt{\hat{v}_i^{(t)}}} \\
					&= z_i^{(t+1)} - \eps_t\frac{1}{1-\eta_1^t}\frac{\eta_1 m_i^{(t-1)} + (1-\eta_1)g_{t,i}}{\sqrt{\hat{v}_i^{(t)}}}.
	\end{align}
	Therefore, according to~\cite{kingma2014adam}, we define $\eta_1=\eta_{1,t}$, allowing $\eta_1$ to change with an increase in the number of iterations, and $\eta_{1,t}$ is non-decreasing. This implies that momentum gradually diminishes, and eventually, $m_i$ approaches $g_i$. Such that, $z_i^{(t+1)}$ changes to:
	\begin{align}
	&	z_i^{(t+1)} = z_i^{(t)} - \gamma_t \frac{\eta_{1,t}m_i^{t-1}+\left(1-\eta_{1,t}\right)g_{t,i}}{\sqrt{\hat{v}_i^{(t)}}} \\
		&\Longrightarrow \left(z_i^{(t+1)} - z_i^*\right)^2 \nonumber \\
		&= \left[\left( z_i^{(t)}  - z_i^*\right) - \gamma_t \frac{\eta_{1,t}m_i^{t-1}+\left(1-\eta_{1,t}\right)g_{t,i}}{\sqrt{\hat{v}_i^{(t)}}}\right]^2 \\
		&\Longrightarrow 2\gamma_t \frac{\eta_{1,t}m_i^{t-1}+\left(1-\eta_{1,t}\right)g_{t,i}}{\sqrt{\hat{v}_i^{(t)}}}\left( z_i^{(t)}  - z_i^*\right) = \nonumber \\
		& \left( z_i^{(t)}  - z_i^*\right)^2 - \left(z_i^{(t+1)} - z_i^*\right)^2 \nonumber \\
		&+\gamma_t \frac{\left[ \eta_{1,t}m_i^{t-1} +  \left(1-\eta_{1,t}\right)g_{t,i}\right] ^2}{\sqrt{\hat{v}_i^{(t)}}}.
	\end{align}
	Here, $\gamma_t = \eps_t \frac{1}{1 - \prod_{s=1}^{t}\eta_{1,s}}$. Considering $m^{(t)}_i=\eta_{1,t}m^{(t-1)} + (1-\eta_{1,t})g_{t,i}$, we have:
	\begin{align}
		& g_{t,i}\left(z^{(t)}_i - z^*_i\right) = \underbrace{\frac{\sqrt{\hat{v}_i^{(t)}}\left[\left( z_i^{(t)}  - z_i^*\right)^2 - \left(z_i^{(t+1)} - z_i^*\right)^2\right]}{2\gamma_t(1-\eta_{1,t})}}_{(1)} \nonumber \\
		& - \underbrace{\frac{\eta_{1,t}}{1 - \eta_{1,t}}m_i^{(t-1)}\left( z_i^{(t)}  - z_i^*\right)}_{(2)} + \underbrace{\frac{\gamma_t}{2(1-\eta_{1,t})}\frac{\left(m_i^{(t)}\right)^2}{\sqrt{\hat{v}_i^{(t)}}}}_{(3)}.
		\label{eq:adam_item}
	\end{align}
	In order to get the upper bound of $R(T)$, we need to rescale $(1)$, $(2)$, and $(3)$ in the Eq.~\ref{eq:adam_item}. For $(1)$, we have:
	\begin{align}
		&\sum_{t=1}^{T} \frac{\sqrt{\hat{v}_i^{(t)}}\left[\left( z_i^{(t)}  - z_i^*\right)^2 - \left(z_i^{(t+1)} - z_i^*\right)^2\right]}{2\gamma_t(1-\eta_{1,t})} \\
	   =&\sum_{t=1}^{T} \frac{\sqrt{\hat{v}_i^{(t)}}\left[\left( z_i^{(t)}  - z_i^*\right)^2 - \left(z_i^{(t+1)} - z_i^*\right)^2\right]}{2  \eps_t \frac{1}{1 - \prod_{s=1}^{t}\eta_{1,s}} (1-\eta_{1,t})} \\
	   =&\sum_{t=1}^{T} \frac{\sqrt{\hat{v}_i^{(t)}}\left[\left( z_i^{(t)}  - z_i^*\right)^2 - \left(z_i^{(t+1)} - z_i^*\right)^2\right]\left(1 - \prod_{s=1}^{t}\eta_{1,s}\right)}{2  \eps_t (1-\eta_{1,t})} \\
	   \leq&\sum_{t=1}^{T} \frac{\sqrt{\hat{v}_i^{(t)}}\left[\left( z_i^{(t)}  - z_i^*\right)^2 - \left(z_i^{(t+1)} - z_i^*\right)^2\right]}{2\eps_t(1-\eta_{1,1})}.
	\end{align}
	We obtain the result using permutation recombination and summation:
	\begin{align}
		&\sum_{t=1}^{T} \frac{\sqrt{\hat{v}_i^{(t)}}\left[\left( z_i^{(t)}  - z_i^*\right)^2 - \left(z_i^{(t+1)} - z_i^*\right)^2\right]}{2\eps_t(1-\eta_{1,1})} \\
		=&\sum_{t=1}^{T} \frac{\sqrt{\hat{v}_i^{(t)}}\left(z_i^{(t)} - z_i^* \right)^2}{2\eps_t(1-\eta_{1,1})} - \frac{\sqrt{\hat{v}_i^{(t)}}\left(z_i^{(t+1)} - z_i^* \right)^2}{2\eps_t(1-\eta_{1,1})} \\
		=&\frac{\sqrt{\hat{v}_i^{(1)}}\left(z_i^{(1)} - z_i^* \right)^2}{2\eps_1 (1-\eta_{1,1})} - \frac{\sqrt{\hat{v}_i^{(T)}}\left(z_i^{(T+1)} - z_i^* \right)^2}{2\eps_T(1-\eta_{1,1})} + \nonumber \\
		& \sum_{t=2}^{T} \left(z_i^{(t)} - z_i^* \right)^2 \cdot \left[\frac{\sqrt{\hat{v}_i^{(t)}}}{2\eps_t(1-\eta_{1,1})} - \frac{\sqrt{\hat{v}_i^{(t-1)}}}{2\eps_{t-1}(1-\eta_{1,1})}\right].
	\end{align}
	We mainly focus on the last term. Considering $\frac{\sqrt{\hat{v}_i^{(t)}}}{2\eps_t(1-\eta_{1,1})} \geq \frac{\sqrt{\hat{v}_i^{(t-1)}}}{2\eps_{t-1}(1-\eta_{1,1})}, \forall t$, we have:
	\begin{align}
		& \sum_{t=2}^{T} \left(z_i^{(t)} - z_i^* \right)^2 \cdot \left[\frac{\sqrt{\hat{v}_i^{(t)}}}{2\eps_t(1-\eta_{1,1})} - \frac{\sqrt{\hat{v}_i^{(t-1)}}}{2\eps_{t-1}(1-\eta_{1,1})}\right] \\
	\leq& \sum_{t=2}^{T} D^2_i \cdot \left[\frac{\sqrt{\hat{v}_i^{(t)}}}{2\eps_t(1-\eta_{1,1})} - \frac{\sqrt{\hat{v}_i^{(t-1)}}}{2\eps_{t-1}(1-\eta_{1,1})}\right] \\
	   =& D_i^2 \left[\frac{\sqrt{\hat{v}_i^{(T)}}}{2\eps_T(1-\eta_{1,1})} - \frac{\sqrt{\hat{v}_i^{(1)}}}{2\eps_{1}(1-\eta_{1,1})}\right].
	\end{align}
	Because $\frac{\sqrt{\hat{v}_i^{(1)}}\left(z_i^{(1)} - z_i^* \right)^2}{2\eps_1 (1-\eta_{1,1})} \leq \frac{D^2_i \sqrt{\hat{v}_i^{(1)}}}{2\eps_1(1-\eta_{1,1})}, \frac{\sqrt{\hat{v}_i^{(T)}}\left(z_i^{(T+1)} - z_i^* \right)^2}{2\eps_T(1-\eta_{1,1})} \leq 0$, we can rescale $(1)$ to:
	\begin{align}
		& \sum_{t=1}^{T} \frac{\sqrt{\hat{v}_i^{(t)}}\left[\left( z_i^{(t)}  - z_i^*\right)^2 - \left(z_i^{(t+1)} - z_i^*\right)^2\right]}{2\gamma_t(1-\eta_{1,t})} \\
		\leq & \left[\frac{^2_i\sqrt{\hat{v}_i^{(T)}}}{2\eps_T(1-\eta_{1,1})} - \frac{D^2_i\sqrt{\hat{v}_i^{(1)}}}{2\eps_1(1-\eta_{1,1})}\right] + \frac{D_i^2 \sqrt{\hat{v}_i^{(1)}}}{2\eps_T(1-\eta_{1,1})} \\
		\leq & \frac{^2_i\sqrt{\hat{v}_i^{(T)}}}{2\eps_T(1-\eta_{1,1})}.
	\end{align}
	We focus on $\hat{v}_i^{(T)}$ for further scaling. As mentioned earlier, $v_i^{(t)} = (1 - \eta_2)\sum_{s=1}^{t}\eta_{2,t-s}g^2_{s,i}$, so we can explore its boundedness based on the earlier gradient bounded assumption:
	\begin{align}
		& \left.\begin{array}{cc}
			v_i^{(t)} & \leq \\
			\hat{v}_i^{(t)} & = \\
		\end{array}\right\rbrace \frac{v_i^{(t)}}{1-\eta_{2}^{t-s}} \leq \frac{(1-\eta_2)\sum_{s=1}^{t}\eta_{2}^{t-s}G_i^2}{1-\eta_{2,t}} \\
		& = \frac{G_i^2(1-\eta_{2}^{t-s})}{1-\eta_{2,t}} = G_i^2.
	\end{align}
	So, finally, $(1)$ is scaled to:
	\begin{align}
		& \sum_{t=1}^{T} \frac{\sqrt{\hat{v}_i^{(t)}}\left[\left( z_i^{(t)}  - z_i^*\right)^2 - \left(z_i^{(t+1)} - z_i^*\right)^2\right]}{2\gamma_t(1-\eta_{1,t})} \\
	\leq& \frac{D^2_i\sqrt{\hat{v}_i^{(T)}}}{2\eps_T(1-\eta_{1,1})} \leq \frac{D_i^2 G_i}{2\eps_T(1-\eta_{1,1})}.
	\end{align}
	Regarding $(2)$, we first apply the variable bounded assumption:
	\begin{align}
		& \sum_{t=1}^{T}\frac{- \eta_{1,t}}{1-\eta_{1,t}}\left(z^{(t)}_i - z^*_i\right) \\
	   =& \sum_{t=1}^T\frac{\eta_{1,t}}{1-\eta_{1,t}}m^{(t-1)}_i \left[-\left(z^{(t)}_i - z^*_i\right)\right] \\
	\leq&  \sum_{t=1}^T\frac{\eta_{1,t}}{1-\eta_{1,t}}m^{(t-1)}_i \vert m_i^{(t-1)} \vert D_i.
	\end{align}
	Now, we focuses on $m_i^{(t-1)}$:
	\begin{align}
		m_i^{(t)} 
		=& \eta_{1,t} m_i^{(t-1)} + (1- \eta_{1,t}) g_{t,i} \\
		=& \eta_{1,t} \eta_{1,t-1} m_i^{(t-2)} + \eta_{1,t}(1-\eta_{1,t-1}) g_{t-1,i} + \nonumber \\ 
		&(1 - \eta_{1,t})g_{t,i} \\
		=& \eta_{1,t} \eta_{1,t-1} \eta_{1,t-2} m_i^{(t-3)} + \eta_{1,t}\eta_{1,t-1}(1-\eta_{1,t-2})g_{t-2,i} \nonumber \\
		& + \eta_{1,t}(1-\eta_{1,t-1}) g_{t-1,i} + (1 - \eta_{1,t})g_{t,i} \\
		=& \eta_{1,t} \eta_{1,t-1} \cdots \eta_{1,1} m_i^{(0)}+\eta_{1,t}\eta_{1,t-1}\cdots(1-\eta_{1,1})g_{1,i} + \nonumber \\
		&\cdots+\eta_{1,t}\eta_{1,t-1}(1-\eta_{1,t-2})g_{t-2,i}+ \nonumber \\ & \eta_{1,t}(1-\eta_{1,t-1}) g_{t-1,i} + (1 - \eta_{1,t})g_{t,i} \\
		\overset{\text{(a)}}{=}& \eta_{1,t}\eta_{1,t-1}\cdots(1-\eta_{1,1})g_{1,i} +\cdots+ \nonumber \\
		& \eta_{1,t}\eta_{1,t-1}(1-\eta_{1,t-2})g_{t-2,i}+ \nonumber \\ 
		& \eta_{1,t}(1-\eta_{1,t-1}) g_{t-1,i} + (1 - \eta_{1,t})g_{t,i} \\
		=& \sum_{s=1}^t (1 - \eta_{1,s}) \left(\prod_{k=s+1}^t \eta_{1,k}\right) g_{s,i}.
	\end{align}
	Where (a) holds because $m_i ^{(0)}=0$. Here, we apply the gradient bounded assumption, for any $t$ we have:
	\begin{align}
		\vert m_i^{(t)} \vert & \leq \sum_{s=1}^t (1-\eta_{1,s}) \left(\prod_{k=s+1}^t \eta_{1,k} \right) \vert g_{s,i} \vert \\
		& \leq \sum_{s=1}^t (1-\eta_{1,s}) \left(\prod_{k=s+1}^t \eta_{1,k} \right) G_i \\
		& = G_i \left(1 - \sum_{s=1}^t \eta_{1,s} \right) \leq G_i.
	\end{align}
	This way, we can scale equation $(2)$:
	\begin{align}
		& \sum_{t=1}^T \frac{-\eta_{1,t}}{1 - \eta_{1,t}} m_i^{(t-1)} \left(z_i^{(t)} - z_i^*\right) \\
   \leq & \sum_{t=1}^T \frac{\eta_{1,t}}{1 - \eta_{1,t}} m_i^{(t-1)} G_i D_i \\
      = & G_i D_i \sum_{t=1}^{T} \frac{\eta_{1,t}}{1-\eta_{1,t}}.
	\end{align}
	About $(3)$, we mainly focus on $\frac{\left(m_i^{(t)}\right)^2}{\sqrt{\hat{v}_i^{(t)}}} = \sqrt{1 - \eta_2^t} \frac{\left(m_i^{(t)}\right)^2}{\sqrt{\hat{v}_i^{(t)}}} \leq \frac{\left(m_i^{(t)}\right)^2}{\sqrt{\hat{v}_i^{(t)}}}$, we have:
	\begin{align}
		m_i^{(t)} &= \sum_{s=1}^t (1 - \eta_{1,s}) \left(\prod_{k=s+1}^t \eta_{1,k} \right) g_{s,i}, \\
		v_i^{(t)} &= (1 - \eta_2) \sum_{s=1}^{t} \eta_2^{t-s}.
	\end{align}
	Then, we transform $\left(m_i^{(t)}\right)^2$ to:
	\begin{align}
		\left(m_i^{(t)}\right)^2 =& \left(\sum_{s=1}^{t} \frac{(1-\eta_{1,s})\left(\prod_{k=s+1}^t\eta_{1,k}\right)}{\sqrt{(1-\eta_2) \eta_2^{t-s}}} \cdot \sqrt{(1-\eta_2)\eta_2^{t-s}}g_{s,i} \right)^2 \\
		=& \sum_{s=1}^t \left(\frac{(1-\eta_{1,s})\left(\prod_{k=s+1}^t\eta_{1,k}\right)}{\sqrt{(1-\eta_2) \eta_2^{t-s}}}\right)^2 \nonumber \\
		& \cdot \sum_{s=1}^t \left( \sqrt{(1-\eta_2)\eta_2^{t-s}}g_{s,i} \right)^2 \\
		=& \sum_{s=1}^t \frac{(1-\eta_{1,s})^2\left(\prod_{k=s+1}^t\eta_{1,k}\right)^2}{(1-\eta_2) \eta_2^{t-s}} \cdot \sum_{s=1}^t (1-\eta_2)\eta_2^{t-s}g_{s,i}^2.
	\end{align}
	So that, we can derive $(3)$ to:
	\begin{align}
		& \sum_{t=1}^T \frac{\gamma_t}{2(1-\eta_{1,t})} \frac{\left(m_i^{(t)}\right)^2}{\sqrt{\hat{v}_i^{(t)}}} \leq \sum_{t=1}^T \frac{\gamma_t}{2(1-\eta_{1,t})} \cdot \\
		&\sum_{s=1}^t \frac{(1-\eta_{1,s})^2\left(\prod_{k=s+1}^{t} \eta_{1,k} \right)^2}{(1-\eta_2)\eta_2^{t-s}} \cdot \frac{\sum_{s=1}^t (1-\eta_2)\eta_2^{t-s}g_{s,i}^2}{\sqrt{\sum_{s=1}^t (1-\eta_2)\eta_2^{t-s}g_{s,i}^2}} \\
	   =&\sum_{t=1}^{T} \frac{\gamma_t}{2(1-\eta_{1,t})} \sum_{s=1}^t \frac{(1-\eta_{1,s})^2\left(\prod_{k=s+1}^t\eta_{1,k}\right)^2}{(1-\eta_2)\eta_2^{t-s}}\sqrt{v_i^{(t)}} \\
	\leq&\sum_{t=1}^T \frac{\gamma_t}{2(1-\eta_{1,t})} \sum_{s=1}^t \frac{(1-\eta_{1,s})^2\left(\prod_{k=s+1}^t\eta_{1,k}\right)^2}{(1-\eta_2)\eta_2^{t-s}} \cdot G_i.
	\end{align}
	Finally, we have the upper bound of $R(T)$:
	\begin{align}
		R(T) & \leq \frac{\sum_{i=1}^d D_i^2 G_i}{2\eps_T(1-\eta_{1,1})} + \sum_{i=1}^{d} G_i D_i \sum_{t=1}^T \frac{\eta_{1,t}}{1-\eta_{1,t}} \nonumber \\
		& + \sum_{i=1}^{d} G_i \sum_{t=1}^T \frac{\gamma_t}{2(1-\eta_{1,t})} \cdot \nonumber \\
		& \sum_{s=1}^{t} \frac{(1-\eta_{1,s})^2\left(\prod_{k=s+1}^t \eta_{1,k}\right)}{(1-\eta_2)\eta_2^{t-s}} \\
		& = \frac{\sum_{i=1}^{d}D^2_i G_i}{2\eps_T(1-\eta_{1,1})} + \left(\sum_{i=1}^{d}D_iG_i\right)\left(\sum_{t=1}^{T}\frac{\eta_{1,t}}{1-\eta_{1,t}}\right) + \nonumber \\
		& \left(\sum_{i=1}^{d}G_i\right)\left[\sum_{t=1}^{T}\frac{\gamma_t}{2(1-\eta_{1,t})} \cdot \right. \nonumber \\
		& \left. \sum_{s=1}^{t}\frac{(1-\eta_{1,s})^2\left(\prod_{r=s+1}^t\eta_{1,r}\right)^2}{(1-\eta_2)\eta_2^{t-s}}\right],
	\end{align}
	Clearly, when $T \to \infty$, the above equation is tend to $0$:
	\begin{align}
		\lim_{T \to \infty} \frac{R(T)}{T} & \leq \lim_{T \to \infty} \frac{1}{T} \left[ \frac{\sum_{i=1}^{d}D^2_i G_i}{2\eps_T(1-\eta_{1,1})} \right. + \nonumber \\
		&\left(\sum_{i=1}^{d}D_iG_i\right)\left(\sum_{t=1}^{T}\frac{\eta_{1,t}}{1-\eta_{1,t}}\right) + \nonumber \\
		&\left(\sum_{i=1}^{d}G_i\right)\left[\sum_{t=1}^{T}\frac{\gamma_t}{2(1-\eta_{1,t})} \cdot \right. \nonumber \\
		& \left. \sum_{s=1}^{t}\frac{(1-\eta_{1,s})^2\left(\prod_{r=s+1}^t\eta_{1,r}\right)^2}{(1-\eta_2)\eta_2^{t-s}}\right] = 0.
	\end{align}
	 Thus, we can conclude that the Improved-Momentum-variant history gradient update algorithm is tend to convergence. Considering $\eta_{1,t} \in (0, 1), \forall t, \eta_{1,1} \geq \eta_{1,2} \geq \dots \geq \eta_{1,T} \geq \dots$ and $\eta_2 \in (0,1), \frac{\eta_{1,t}}{\sqrt{\eta_2}}\leq\sqrt{c} < 1$, about $(2)$, we have:
	 \begin{align}
	 	& \left(\sum_{i=1}^{d}D_iG_i\right)\left(\sum_{t=1}^{T}\frac{\eta_{1,t}}{1-\eta_{1,t}}\right) \nonumber \\
	 \leq&
	     \left(\sum_{i=1}^{d}D_iG_i\right)\left(\frac{1}{1-\eta_{1,t}}\sum_{t=1}^{T}\eta_{1,t}\right).
	 \end{align}
	 About $(3)$, we have:
	 \begin{align}
	 	& \left(\sum_{i=1}^{d}G_i\right)\left[\sum_{t=1}^{T}\frac{\gamma_t}{2(1-\eta_{1,t})} \cdot \sum_{s=1}^{t}\frac{(1-\eta_{1,s})^2\left(\prod_{r=s+1}^t\eta_{1,r}\right)^2}{(1-\eta_2)\eta_2^{t-s}}\right] \\
	   =& \left(\sum_{i=1}^{d}G_i\right) \cdot \sum_{t=1}^T \frac{\eps_t}{2(1-\eta_{1,t})\left(1-\prod_{s=1}^t\eta_{1,s}\right)} \cdot \nonumber \\
	    &\sum_{s=1}^t \left(\frac{1-\eta_{1,s}}{\sqrt{1-\eta_2}}\right)^2 \cdot \prod_{k=s+1}^t\left(\frac{\eta_{1,k}}{\sqrt{\eta_2}}\right) \\
	   \overset{\text{(a)}}{\leq}& \left(\sum_{i=1}^{d}G_i\right) \cdot \sum_{t=1}^T \frac{\eps_t}{2(1-\eta_{1,t})^2 (1-\eta_2)} \sum_{s=1}^t \prod_{k=s+1}^t c \\
	   =& \left(\sum_{i=1}^{d}G_i\right) \cdot \sum_{t=1}^T \frac{\eps_t}{2(1-\eta_{1,t})^2(1-\eta_2)(1-c)} \\
	   =& \left(\sum_{i=1}^{d}G_i\right) \cdot \frac{  \sum_{t=1}^T \eps_t}{2(1-\eta_{1,t})^2(1-\eta_2)(1-c)}.
	 \end{align}
	Where (a) holds because the assumption $\frac{\eta_{1,t}}{\sqrt{\eta_2}}\leq\sqrt{c} < 1$. 	Similar to Theorem~\ref{proofofgd}, when $n=1/2$, $R(T)$ attain the optimal upper bound $\gO\left(T^{1/2}\right)$.
\end{proof}

\section*{Experimental Details}
\label{app:experiment}
\subsection{Model Details}
Here, we present the parameters utilized in our experiments for each dataset.
\begin{table*}[t]
	\centering
	\caption{Model hyperparameters for the unconditional DDPM and LDMs used in the AAPM CT evaluation experiments. All models are trained on a single RTX 4090.}
	\label{tab:hyperparameters}
	\setlength{\tabcolsep}{18pt}
	\begin{tabular}{@{}lccc@{}}
			\toprule
			& \textbf{DDPM} $256\time256$   & \textbf{LDM} $256\time256$    & \textbf{LDM}  $512\time512$      \\ \midrule
			Latent shape             & -       & $32 \times 32 \times 4$     & $32 \times 32 \times 4$         \\
			Input shape              & $256 \times 256$     & $32 \times 32$      & $32 \times 32$         \\
			Diffusion steps          & 1000    & 1000    & 1000       \\
			Noise schedule           & linear  & linear  & linear     \\
			U-Net Channels           & 128     & 128     & 128        \\
			U-Net Channel Multiplier & 1,2,2,4 & 1,2,2,4 & 1,2,2,4    \\
			U-Net Depth              & 2       & 2       & 2          \\
			U-Net Batch size         & 4       & 32      & 32         \\
			U-Net Epochs             & 500     & 500     & 500        \\
			U-Net Learning Rate      & 1e-4    & 1e-4    & 1e-4       \\
			VQVAE Channels           & -       & 128     & 32         \\
			VQVAE Channel Multiplier & -       & 1,1,2,4 & 1,4,8,8,16 \\
			VQ Dimension             & -       & 4       & 4          \\
			Numbers of VQ embedding  & -       & 16384   & 16384      \\
			VQVAE Batch size         & -       & 4       & 4          \\
			VQVAE Epochs             & -       & 251     & 251        \\
			VQVAE Learning rate      & -       & 4.5e-5  & 4.5e-5     \\ \bottomrule
		\end{tabular}
\end{table*}

\begin{itemize}
	\item \textbf{CelebaHQ. } This pretrained DDPM model can be accessed from the \href{https://huggingface.co/google/ddpm-ema-celebahq-256}{\textcolor{blue}{huggingface model zoo}}. The pretrained LDM model can be accessed from the \href{https://huggingface.co/CompVis/ldm-celebahq-256}{\textcolor{blue}{huggingface model zoo}}.
	\item \textbf{LSUN Bedroom. } This pretrained DDPM model can be accessed from the \href{https://huggingface.co/google/ddpm-ema-bedroom-256}{\textcolor{blue}{huggingface model zoo}}. The pretrained LDM model can be accessed from the \href{https://github.com/CompVis/latent-diffusion#pretrained-ldms}{\textcolor{blue}{Github repository}}.
	\item \textbf{AAPM CT. } Both the DDPM and LDM models are trained from scratch using the AAPM CT dataset. The hyperparameters employed for the models can be found in Table~\ref{tab:hyperparameters}.
\end{itemize}
All models were trained using a single RTX 4090 GPU, with the training process taking approximately 1 day, 2 days, and 5 days for each respective model. The reconstruction process running on a single RTX 4090 GPU.

\subsection{Details of the measurement operators}
\textbf{CT Reconstruction. } The measurement operator of CT reconstruction can be defined as:
\begin{align}
	\y &= \Ac\x + \n \\
	\y &= \mathbf{P}(\mathbf{\Lambda})\x + \n,
\end{align}
here, we define $\mathbf{P}$ as the discretized Radon transform, which is utilized in CT to generate Sinogram data. Additionally, $\mathbf{\Lambda}$ represents the selection mask matrix, determining the chosen views for measurement. Consequently, the disparity between sparse-view and limited-angle CT reconstruction lies solely in the variation of $\mathbf{\Lambda}$. Throughout our experiments, we set the noise variable $\n$ to 0.

\textbf{Random Inpainting. } The measurement operator of random inpainting can be defined as:
\begin{align}
	\y &= \Ac \x + \n \\
	\y &= \mathbf{M} \circ \x + \n , \n \sim \Nc(\bm{0}, \Ib),
\end{align}
where we define $\mathbf{M}$ as a random masking matrix with the same shape as $\x$ and comprising elementary unit vectors. The symbol $\circ$ represents the Hadamard product. In our experiments, we set $\sigma$ to a value of 0.05.

\textbf{Super Resolution. } The measurement operator of super-resolution can be defined as:
\begin{align}
	\y &= \Ac \x + \n \\
	\y &= \mathbf{H} \x + \n , \n \sim \Nc(\bm{0}, \Ib),
\end{align}
where we denote $\mathbf{H}$ as the bilinear downsampling catalecticant matrix. In our experiments, we set the value of $\sigma$ to 0.05.

The hyperparameters used for the LHGU guiding process with different measurement operators are presented in Table~\ref{tab:exp_params}.

\begin{table*}[t]
	\centering
	\caption{Guidance hyperparameters of LHGU for AAPM CT and nature image experiments.}
	\label{tab:exp_params}
		\begin{tabular}{@{}lcccc@{}}
				\toprule
				& \multicolumn{1}{c} {\textbf{Sparse view CT}} & \multicolumn{1}{c}{\textbf{Limited angle CT}} & \multicolumn{1}{c}{\textbf{Inpainting}} & \multicolumn{1}{c}{\textbf{Super-Resolution}} \\ \midrule
				Evaluation function $\mathcal{U}$           & $L1$                                          & $L2$                                            & $L2$                                      & $L2$                                            \\
				Guidance rate $\epsilon$                  & 0.5                                         & 0.1                                           & 0.05                                    & 0.001                                         \\
				History gradient update & Improved-Momentum-variant                                   & Improved-Momentum-variant                                     & Improved-Momentum-variant                               & Improved-Momentum-variant                                     \\ \bottomrule
			\end{tabular}
\end{table*}

\end{document}